%% file: main.tex
\documentclass{article}
\PassOptionsToPackage{numbers, compress}{natbib}
\PassOptionsToPackage{table}{xcolor}

\usepackage[preprint]{neurips_2026}

\input{preamble}

\title{COMPOSE: Hypergraph Cover Optimization for Multi-view 3D Human Pose Estimation}

\hypersetup{
  pdftitle={COMPOSE: Hypergraph Cover Optimization for Multi-view 3D Human Pose Estimation},
  pdfauthor={Tony Danjun Wang, Tolga Birdal, Nassir Navab, Lennart Bastian},
}

\author{%
  Tony Danjun Wang$^{1,2}$\thanks{Correspondence to: \texttt{tony.wang@tum.de}}
  \And
  Tolga Birdal$^{3}$
  \And
  Nassir Navab$^{1,2}$
  \And
  Lennart Bastian$^{1,2,3}$
  \AND
  \mdseries
  $^{1}$School of Computation, Information, and Technology, Technical University of Munich, Germany \\
  $^{2}$Munich Center for Machine Learning, Germany \\
  $^{3}$Department of Computing, Imperial College London, United Kingdom
}

\begin{document}

\maketitle

\input{corpus/0_main}

\bibliographystyle{plainnat}
\bibliography{main.bib}

\newpage
\appendix
\input{supplementary/0_main}

\end{document}

%% file: preamble.tex
\usepackage[utf8]{inputenc} %
\usepackage[T1]{fontenc}    %

\usepackage{
  amsmath,
  amssymb,
  amsfonts,
  amsthm,
  mathtools,
  booktabs,
  nicefrac,
  microtype,
  url,
  xspace,
  enumitem,
  multirow,
  tabularx,
  dsfont,
  etoolbox,
  makecell
}
\usepackage[table]{xcolor} %
\usepackage{booktabs}      %
\usepackage{subfig}
\usepackage{graphicx}
\usepackage{tikz}
\usetikzlibrary{calc,backgrounds,arrows.meta}
\usepackage{siunitx}
\usepackage[ruled,vlined,linesnumbered]{algorithm2e}

\usepackage{hyperref}
\usepackage[capitalize]{cleveref}

\newtheorem{definition}{Definition}
\newtheorem{proposition}{Proposition}

\crefname{algocf}{algorithm}{algorithms}
\Crefname{algocf}{Algorithm}{Algorithms}

\crefname{proposition}{proposition}{propositions}
\Crefname{proposition}{Proposition}{Propositions}

\crefname{definition}{definition}{definitions}
\Crefname{definition}{Definition}{Definitions}

\robustify\textbf

\newcommand{\name}{COMPOSE\xspace}

\renewcommand{\paragraph}[1]{{\vspace{1mm}\noindent\bfseries #1}.}

\definecolor{colorFull}{RGB}{0, 50, 150}    %
\definecolor{colorSelf}{RGB}{200, 80, 0}    %
\definecolor{colorOpt}{RGB}{0, 100, 50}     %

\newcommand{\bestFull}{\color{colorFull}\bfseries}
\newcommand{\bestSelf}{\color{colorSelf}\bfseries}
\newcommand{\bestOpt}{\color{colorOpt}\bfseries}

\newcommand{\cU}{\mathcal{U}}
\newcommand{\cE}{\mathcal{E}}

\newcommand{\bx}{\mathbf{x}}

\DeclareMathOperator*{\argmax}{arg\,max}
\DeclareMathOperator*{\argmin}{arg\,min}

%% file: corpus/0_main.tex
\input{corpus/1_abstract}

\input{assets/figures/teaser.tex}

\input{corpus/2_introduction}
\input{corpus/4_methods}
\input{corpus/5_experiments_and_results}

\input{corpus/6_conclusion}

\input{corpus/99_epilogue}

%% file: corpus/1_abstract.tex
\begin{abstract}
    3D human pose estimation from sparse multi-view camera rigs is an essential task for numerous applications, including action recognition, sports analysis, and human-robot interaction.
    While learned methods dominate the field on benchmarks, they require large annotated datasets; training-free optimization-based methods remain promising as they circumvent 3D supervision by solving a correspondence problem across views from 2D detections.
    Existing combinatorial formulations rely on pairwise associations to model this correspondence problem and enforce global consistency across views only as a downstream constraint.
    However, reconciling locally plausible pairwise matches becomes brittle under occlusion and noisy detections, where local errors propagate globally.
    We propose \textsc{\name}, which recasts multi-view 3D human pose estimation as a weighted exact-cover optimization over a hypergraph of person hypotheses.
    Our formulation replaces pairwise association and post-hoc consistency enforcement with a single global combinatorial objective.
    To address the exponentially large candidate space, we introduce a geometric pruning strategy alongside two complementary solvers: an exact Integer Linear Programming formulation and a scalable relaxation via Belief Propagation.
    Without any 3D supervision, \name improves average precision by up to 31 points over the best optimization-based method and 13 points over self-supervised learned methods, demonstrating the effectiveness of higher-order combinatorial association for training-free multi-view 3D human pose estimation.
\end{abstract}

%% file: assets/figures/teaser.tex
\begin{figure}[ht]
    \centering
    \includegraphics[width=1.0 \textwidth]{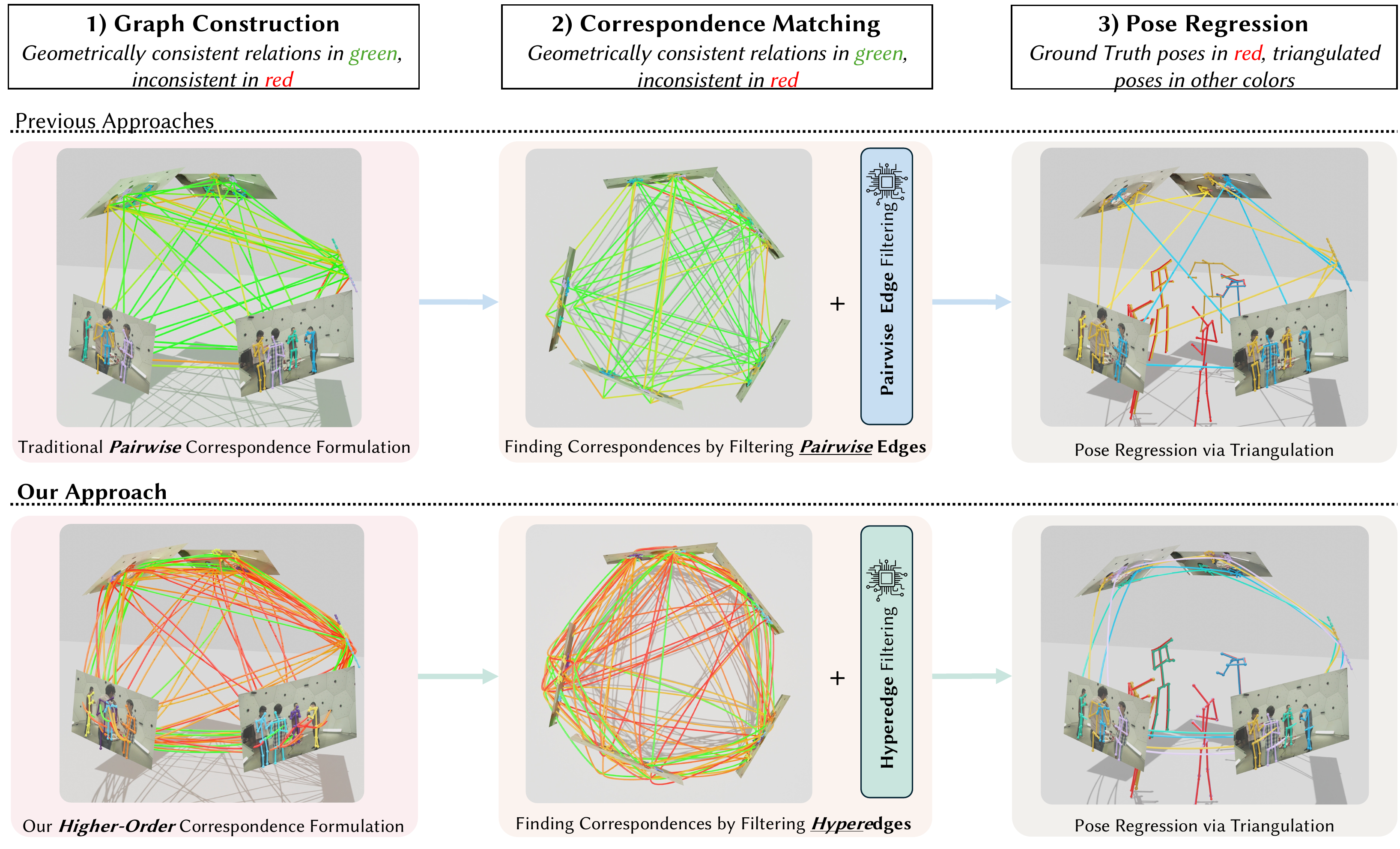}
    \caption{
    \textbf{Top:} Traditional approaches rely on pairwise geometric constraints 
    \cite{dong2019fast,wu_GraphBased3DMultiPerson_2021}. 
    As illustrated, these methods generate pairwise associations that, 
    while \textit{locally} consistent between two views, often fail to form a 
    \textit{globally} coherent structure. Consequently, algorithms face the 
    difficult task of reconciling these locally plausible but globally conflicting 
    edges via \textit{synchronization } to recover the correct 3D poses.
    \textbf{Bottom:} We propose \textsc{\name}, a hypergraph formulation that jointly models 
    higher-order relationships across views. We re-frame correspondence matching 
    as a \textit{hypergraph partitioning} problem, where hyperedges encode 
    multi-view consistency. 
    This global formulation effectively resolves ambiguities by enforcing consensus across the entire set of views.
    }
    \label{fig:teaser}
    \vspace{-5.0mm}
\end{figure}

%% file: corpus/2_introduction.tex
\section{Introduction}
\label{sec:introduction}
\vspace{-2mm}

Human pose estimation is a fundamental task in computer vision, yet its deployment in safety-critical scenarios remains challenging. 
While significant progress has been made in monocular 2D settings due to the availability of large-scale annotated datasets \cite{coco_dataset,martinez2017simple,zhu2023motionbert}, these methods inherently lack metric depth information and can be prone to occlusions \cite{zheng2023deep}. 
In real-world applications such as collaborative human-robot interaction \cite{goodrich2008human} and operating room monitoring \cite{wang2025beyond}, precise spatial localization thus requires lifting 2D detections across multiple camera views into a joint 3D coordinate system \cite{dong2019fast,panoptic}. 
By capturing activities from diverse viewpoints, these setups enable triangulation, thereby resolving depth ambiguities and noise inherent to single-view imaging.

Even with advances in 2D backbones, the community has largely shifted toward end-to-end learning-based methods for multi-view 3D human pose estimation \cite{tu_VoxelPoseMulticamera3D_2020,ye_FasterVoxelPoseRealtime_2022,liao_MultipleViewGeometry_2024}.
These approaches have significantly advanced the field, but require large annotated datasets.
However, obtaining ground-truth 3D annotations is labor-intensive and technically challenging, as skeletal joints are internal to the human body and cannot be directly observed or accurately annotated from surface images alone. 
Furthermore, learning-based models frequently suffer from domain gaps, failing to generalize to unseen environments or novel camera configurations
\cite{liao_MultipleViewGeometry_2024,chharia_MVSSMMultiViewState_2025,srivastav_SelfPose3dSelfSupervisedMultiPerson_2024}. 
Domain-agnostic methods that do not rely on 3D supervision are thus highly sought after.

Training-free optimization techniques offer a complementary paradigm.
Rather than amortizing inference cost over a training distribution by learning model parameters, these methods explicitly solve the multi-view assignment problem at test time for each instance, without extensive supervision or annotations. 
As such, they do not require domain-specific 3D pose detectors and can encode geometric and combinatorial constraints, such as reprojection consistency and multi-person exclusivity, directly in the objective. 
Moreover, they preserve a modular separation between 2D perception and 3D reconstruction: improved off-the-shelf keypoint detectors can be incorporated without retraining the geometric solver. 
These properties are particularly appealing for deployments with novel camera layouts, limited access to 3D ground truth, or environments that exhibit a significant distribution shift from the training data of learned 3D pose models.

We present \textsc{\name}, a hypergraph formulation for 3D human pose estimation that models multi-view relations between \textit{2D observations} as a weighted exact-cover over a hypergraph (see \cref{fig:teaser}). 
Existing methods model multi-view relationships by \textit{synchronizing} pairwise relations to recover a cycle-consistent matching between views~\cite{dong2019fast,birdal2019probabilistic,chen2025learning} (see \cref{sec:related_works} for a comprehensive overview).
However, these methods are notoriously reliant on pairwise matches computed in isolation (\cref{fig:teaser}, top), \textbf{making it challenging to resolve ambiguities when views are occluded or noisy} \cite{han2022multi}.

To alleviate these challenges, \textsc{\name} extends modeling beyond pairwise relationships by abstracting 2D image correspondence in a hypergraph formulation, where each hyperedge represents a candidate multi-view person hypothesis.
By scoring sets of 2D detections simultaneously, \name enforces a holistic consensus across all cameras, improving robustness to outliers in individual views (see \cref{fig:teaser}, where hyperedges better express global geometric consistency).
To solve this higher-order partitioning problem, we propose two complementary optimization strategies: an exact Integer Linear Programming (ILP) formulation and a scalable probabilistic relaxation using loopy Belief Propagation (BP).
The ILP formulation recovers the globally optimal assignment, while the BP relaxation produces continuous marginal beliefs over hyperedges.
This enables fully parallelizable, GPU-accelerated inference and yields soft confidence scores over the associations, providing an uncertainty-aware alternative to hard assignments.
With these optimization-based approaches, the proposed method successfully reconstructs accurate 3D poses from multi-view 2D observations, even in the presence of severe occlusions.
Furthermore, experiments demonstrate that our hypergraph-based approach outperforms state-of-the-art optimization-based and recent self-supervised learning methods.

Our main contributions can be summarized as follows:
\begin{itemize}[leftmargin=*, itemsep=0.2em, topsep=0.2em, parsep=0pt, partopsep=0pt]
    \item We cast multi-view 3D human pose estimation as weighted exact-cover over a hypergraph of person hypotheses, making the multi-view hypothesis --- rather than the pairwise match --- the atomic unit of association.
    \item We introduce a combinatorial optimization objective with a geometric pruning strategy, enabling recovery of the globally optimal cover via Integer Linear Programming.
    \item We derive a probabilistic relaxation of the exact-cover objective and solve it with loopy Belief Propagation, yielding parallelizable GPU-accelerated inference and continuous marginal beliefs.
    \item We show \textsc{\name} improves over state-of-the-art optimization-based baselines under identical 2D detections and surpasses recent self-supervised methods without any 3D supervision.
\end{itemize}

%% file: corpus/4_methods.tex
\input{assets/figures/pipeline.tex}

\section{Methodology}
\label{sec:methodology}
\vspace{-2mm}

\textsc{\name} addresses multi-view multi-person 3D pose estimation by decomposing the task into 2D detection, higher-order association, and triangulation, as illustrated in \cref{fig:pipeline}. 
2D poses are first independently detected in each camera view using an off-the-shelf pose estimator.
The central modeling choice in \textsc{\name} is to treat an entire candidate 3D person hypothesis as the atomic association, rather than pairwise matching. 
Pairwise edges express only local compatibility between detections; a hyperedge can group detections across multiple views and represents the hypothesis that they are projections of the same physical individual.
This turns correspondence recovery into a \textit{weighted exact-cover problem} over a hypergraph: the solver selects a set of mutually exclusive hypotheses that explains the observed 2D detections, with scores derived from geometric consistency.

\subsection{Problem Setting and Hypergraph Construction}
\label{subsec:hypergraph_formulation}
\vspace{-2mm}

We consider $V$ calibrated RGB cameras with projection functions $\{\pi_v\}_{v=1}^{V}$.
For each view $v$, a 2D pose detector produces pose detections $\cU_v = \{U_i^v\}_{i=1}^{n_v}$.
Each detection $U_i^v$ contains $J$ image-space joints, with joint $j$ denoted by $\mathbf{u}_{i,j}^v \in \mathbb{R}^{2}$.
The set of all 2D pose detections is $\cU = \bigcup_{v=1}^{V} \cU_v$.

We model multi-view correspondences as a hypergraph $\mathcal{G}=(\cU,\cE)$, where the vertex set $\cU$ consists of all 2D pose detections and each hyperedge $e\in\cE$ is a candidate correspondence group.
A valid hyperedge contains at most one detection per camera view: a valid correspondence should have \textbf{at most one detection per image}:
\[
    |e \cap \cU_v| \leq 1, \qquad \forall v,
\]

We denote $\cE_{\mathrm{all}}$ as the set of all non-empty hyperedges satisfying this constraint.
A non-singleton hyperedge represents one candidate 3D person hypothesis: all its 2D detections are hypothesized projections of the same physical person.
For instance, $e=\{U_2^1,U_5^3,U_1^4\}$ groups one detection from each of views $1$, $3$, and $4$ into a single multi-view correspondence hypothesis.
Singleton hyperedges are retained to explain unmatched, false-positive, or single-view detections.
In practice, \textsc{\name} operates on a pruned hypergraph with $\cE\subseteq\cE_{\mathrm{all}}$ (see \cref{subsec:hyperedge_scoring}).

\subsection{Weighted Exact-Cover Optimization}
\label{subsec:ilp}
\vspace{-2mm}

Given a candidate hyperedge set $\cE$ and a compatibility score $s(e)$ for each hyperedge, we propose to model multi-person association as a \textit{weighted exact-cover problem} \cite{Karp1972Reducibility}, optimizing over a disjoint subset of hyperedges, ensuring that every 2D pose detection is explained exactly once.

\begin{definition}[Weighted Exact-Cover ILP]
\label{def:ilp}
Let $\bx=(x_e)_{e\in\cE}$ with $x_e\in\{0,1\}$ indicate whether hyperedge $e$ is selected.
\textsc{\name} solves
\begin{align}
    \label{eq:ilp}
    \max_{\bx}\quad
    & \sum_{e \in \cE} (s(e)-\gamma)x_e \\
    \mathrm{s.t.}\quad
    & \sum_{e:\,u\in e}x_e = 1,
    \quad \forall u\in\cU, \\
    & x_e\in\{0,1\},
    \quad \forall e\in\cE.
\end{align}
Here, $s(e)$ is a hyperedge compatibility score, and $\gamma>0$ penalizes the number of selected hyperedges, thereby favoring compact explanations that group geometrically consistent detections across views.
\end{definition}

The resulting association is the selected cover
$\cE^{\star}=\{e\in\cE:x_e=1\}$, comprising vertex-disjoint hyperedges that jointly cover all observed 2D detections.

\subsection{Hyperedge Scoring and Candidate Pruning}
\label{subsec:hyperedge_scoring}
\vspace{-2mm}

\paragraph{Hyperedge Scoring}
To realize the ILP, we now specify the compatibility score $s(e)$ and prune $\cE_{\mathrm{all}}$ to a tractable candidate set $\cE$.
For calibrated cameras, we score non-singleton hyperedges by reprojection consistency.
For each candidate hyperedge $e$ with $|e|\geq 2$, let $\hat{\mathbf{y}}_j(e)$ be the DLT triangulation of joint $j$ from the detections in $e$.
We define the reprojection cost
\[
    \mathcal{C}(e)
    =
    \frac{1}{J|e|}
    \sum_{j=1}^{J}
    \sum_{U_i^v \in e}
    \left\|
        \pi_v(\hat{\mathbf{y}}_j(e)) - \mathbf{u}_{i,j}^{v}
    \right\|^2.
\]
The corresponding compatibility score is
\[
    s(e) = \exp(-\lambda \cdot \mathcal{C}(e)),
\]
where $\lambda$ controls the sensitivity to reprojection error.
Singletons receive a fixed prior score $s_{\texttt{single}}$.

The complete candidate set $\cE_{\mathrm{all}}$ contains all combinations of detections across non-empty subsets of views.
We observe that the size of this problem grows exponentially with the number of views:
\begin{proposition}[Number of candidate hyperedges]
\label{prop:numhe}
The total number of potential hyperedges $M$ is given by:
\[
M = \sum_{\emptyset \neq S \subseteq \{1,\dots,V\}} \prod_{v \in S} n_v = \prod_{v=1}^{V}(1 + n_v) - 1
\]
where $S$ represents a subset of views and $n_{v} = |\cU_v|$ is the number of detected poses in view $v$.
This implies exponential growth with respect to the number of views $V$, yielding a complexity $\mathcal{O}((N+1)^{V})$, where $N=\max_v n_v$ denotes the maximum number of detected poses in any single view.
\end{proposition}

\begin{proof}
For each view $v$, a valid hyperedge either selects one of the $n_v$ detections or selects no detection from that view, giving $1+n_v$ choices.
Multiplying over all views and subtracting the all-empty choice yields
\[
    M = \prod_{v=1}^{V}(1+n_v)-1
    \leq (1+N)^V - 1
    \in \mathcal{O}\bigl((N+1)^V\bigr),
    \quad N=\max_v n_v.
\]
\end{proof}

\paragraph{Candidate Pruning}
\Cref{prop:numhe} shows that naive optimization over $\cE_{\mathrm{all}}$ is intractable as the number of views or detections grows.
We therefore construct the candidate set $\cE$ used in \eqref{eq:ilp} by retaining (i) all singleton hyperedges, ensuring that the exact-cover constraints remain feasible, and (ii) multi-view hyperedges satisfying the geometric consistency criterion
\[
    \mathcal{C}(e) \leq \tau,
\]
yielding a tractable set of physically plausible hypotheses.
While we use reprojection error for calibrated cameras, our formulation is not tied to this specific metric: in weakly calibrated or uncalibrated settings, $\mathcal{C}(e)$ could instead be based on epipolar or trifocal consistency~\cite{hartley2003multiple,li2024multi,huang2021dynamic}, and the compatibility score could incorporate off-the-shelf appearance descriptors when additional visual evidence is beneficial~\cite{zhou2021learning,oquab2024dinov2}; we provide experiments in the supplementaries (\cref{app_sec:ablation_cost_functions}).

\paragraph{Solving the ILP}
Although weighted exact-cover is NP-hard \cite{Karp1972Reducibility}, the candidate pruning above substantially reduces the candidate set, enabling effective solution of \eqref{eq:ilp} via an ILP solver with branch-and-cut~\cite{forrest2005cbc}, as shown in \cref{fig:runtime_analysis}.
The ILP returns the globally optimal cover over $\cE$, after which selected non-singleton hyperedges are triangulated into 3D poses (see \cref{app_sec:weighted_triangulation} for details).

\subsection{Probabilistic Relaxation via Belief Propagation}
\label{sec:relaxation}
\vspace{-2mm}

While the ILP returns a globally optimal hard assignment over the retained candidate set $\cE$, its worst-case complexity remains exponential despite geometric pruning.
Moreover, the ILP solution does not provide marginal association uncertainty, which is useful for downstream tasks such as tracking~\cite{aharon2022bot}.
We therefore construct a probabilistic relaxation that associates each binary decision variable with a \textit{continuous marginal belief}, which can be evaluated efficiently via message passing.

The key idea is to reinterpret the ILP objective as the energy of a Gibbs distribution over the binary assignment vector $\bx=(x_e)_{e\in\cE}$.
The resulting factor graph defines a Markov random field over hyperedge-selection variables, and has the same MAP solution as the ILP under exact coverage constraints, while the relaxed version used for BP yields soft association beliefs that can be rounded into a discrete selection.
We provide guiding intuitions in the supplementary material (\cref{app_sec:additional_proofs}) and refer the reader to~\cite{yedidia2001bethe,yedidia2003understanding} for the fundamentals of belief propagation.

\begin{proposition}[Exact MAP--ILP equivalence]\label{prop:map_ilp}
Define unary factors
\begin{equation}\label{eq:unary}
  f_e(x_e) := \exp\bigl(\beta\,(s(e) - \gamma)\, x_e\bigr),
\end{equation}
constraint factors
\begin{equation}\label{eq:exclusion_exact}
  g_u\bigl(\bx_{\cE(u)}\bigr)
  =
  \begin{cases}
    1, & \sum_{e \in \cE(u)} x_e = 1,\\[2pt]
    0, & \text{otherwise},
  \end{cases}
\end{equation}
and the Gibbs distribution
\begin{equation}
  \label{eq:gibbs}
  p_\beta(\bx)
  = \frac{1}{Z}
    \prod_{e \in \cE} f_e(x_e)\;
    \prod_{u \in \cU} g_u\!\bigl(\bx_{\cE(u)}\bigr),
\end{equation}
where $Z$ is a normalization constant. Then, for any $\beta > 0$, the MAP solution of \eqref{eq:gibbs} coincides with the optimizer of the ILP \eqref{eq:ilp}.
\end{proposition}

\begin{proof}[Proof sketch]
We observe that due to monotonicity of the objective function for $\bx \in \{0, 1\}$,
$\argmax_{\bx} p_\beta(\bx) = \argmax_{\bx} \log p_\beta(\bx)$.
The constraint factors \eqref{eq:exclusion_exact} enforce
$g_u(\bx_{\cE(u)}) \in \{0,1\}$, sending any infeasible assignment
(i.e., $\sum_{e \ni u} x_e \neq 1$ for some $u$) to $-\infty$ in
log-space, so the feasible set is exactly that of the ILP.
On this feasible set,
$\log p_\beta(\bx) = \beta \sum_{e \in \cE}(s(e)-\gamma)\,x_e + \mathrm{const}$.
Since $\beta > 0$ is a positive scalar, it does not affect the
$\argmax$, which thus coincides with the ILP \eqref{eq:ilp}.
We provide a formal analysis in the supplementary materials.
\end{proof}

\subsection{Loopy Belief Propagation On The Factor Graph}
\label{ssec:bp}
\vspace{-2mm}

Exact inference on the factor graph induced by \eqref{eq:gibbs} is intractable due to the loopy structure created by hyperedges that share detections.
Thus, we use loopy belief propagation (BP) \cite{yedidia2003understanding}, which iteratively passes messages between variable nodes (hyperedges $e \in \cE$) and factor nodes (detections $u \in \cU$).

\paragraph{Relaxation Used For BP}
To obtain a softer and more robust model for loopy BP, we relax the equality
constraint with a tunable uncovered-node penalty:
\begin{equation}\label{eq:exclusion}
  g_u^{(\eta)}\!\bigl(\bx_{\cE(u)}\bigr)
  =
  \mathds{1}\!\left[\sum_{e\in\cE(u)}x_e\le 1\right]
  \exp\!\left(
    -\eta\;\mathds{1}\!\left[\sum_{e\in\cE(u)}x_e=0\right]
  \right),
  \qquad \eta\ge 0.
\end{equation}
This defines a continuum:
(i)~$\eta=0$ gives the pure ``at-most-one'' relaxation,
(ii)~$\eta\to\infty$ recovers exact coverage ($=1$) in the limit.
For BP inference, we set $\eta=0$ for computational simplicity, yielding the standard exclusion factor used in our updates.

Since all variables are binary, each message reduces to a single scalar log-ratio.
Let $m_{e \to u}$ denote the variable-to-factor log-ratio message from hyperedge~$e$ to detection~$u$, and $n_{u \to e}$ the factor-to-variable message in the reverse direction:
\begin{align}
  m_{e \to u} \coloneqq \log \frac{\mu_{e \to u}(1)}{\mu_{e \to u}(0)},
  \qquad \qquad n_{u \to e} \coloneqq \log \frac{\nu_{u \to e}(1)}{\nu_{u \to e}(0)} \label{eq:nmdef}
\end{align}
where $\mu_{e \to u}$ and $\nu_{u \to e}$ are the standard BP messages in probability space.

\paragraph{BP Inference}
Define the unary log-potential $\phi_e \coloneqq \beta\,(s(e)-\gamma)$.
The scalar BP update equations for $m_{e\to u}$ and $n_{u\to e}$ are given in the supplementary material
(\cref{app_subsec:bp_updates}). 
After $T$ iterations, the approximate marginal belief is
\[
    b_e = \sigma\bigl(\phi_e + \sum_{u \in e} n_{u \to e}\bigr),
\]
where $\sigma$ is the logistic sigmoid.
The final non-overlapping discrete selection is obtained by greedy rounding in decreasing order of $b_e$.

\paragraph{GPU-Accelerated Inference}
Within each synchronous BP iteration, all variable-to-factor messages can be computed in parallel, followed by all factor-to-variable messages.
In practice, we implement the updates as batched tensor operations in PyTorch \cite{ansel2024pytorch}, enabling GPU acceleration.
The ILP solver, however, relies on branch-and-cut \cite{forrest2005cbc}, which is not easily parallelizable, making \name-BP particularly efficient.

%% file: assets/figures/pipeline.tex
\begin{figure}[t!]
    \centering
    \includegraphics[width=\textwidth]{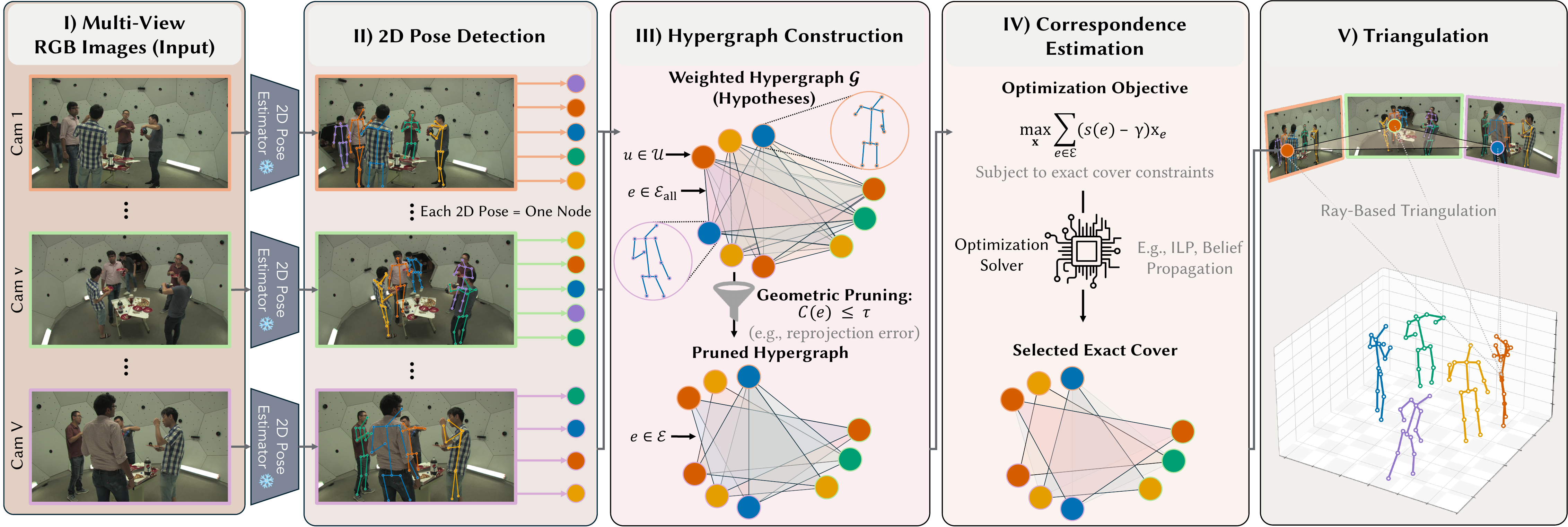}
    \caption{
    \textbf{\name}. 
    Multi-view images are taken as input (I).
    We employ an off-the-shelf 2D pose estimator to extract 2D keypoints (II).
    Next, we construct the weighted hypergraph (III). 
    We then solve the exact cover problem to optimally partition the graph and establish unique correspondences (IV).
    Finally, we triangulate the 3D human poses from the obtained correspondences (V).}
    \label{fig:pipeline}
    \vspace{-5.0mm}
\end{figure}

%% file: corpus/5_experiments_and_results.tex
\input{assets/tables/panoptic_results.tex}

\section{Experiments and Results}
\label{sec:experiments_and_results}
\vspace{-2mm}

\paragraph{Datasets}
We evaluate \textsc{\name} on three public multi-view human pose datasets covering controlled indoor, heavily occluded indoor, and outdoor sparse-camera settings.
CMU Panoptic \cite{panoptic} is used for the main quantitative evaluation and camera-layout generalization experiments, while Shelf and Campus \cite{belagiannis_3DPictorialStructures_2014} evaluate robustness under indoor and fewer camera view settings.

\begin{itemize}[noitemsep,topsep=0em,leftmargin=*]
    \item \textbf{CMU Panoptic~\cite{panoptic}} is a large-scale indoor dataset captured with a dense multi-camera system. 
    Following prior work
    we use the standard evaluation protocol with cameras \texttt{3}, \texttt{6}, \texttt{12}, \texttt{13}, and \texttt{23}.
    
    \item \textbf{Shelf~\cite{belagiannis_3DPictorialStructures_2014}} captures four people interacting in a small indoor environment with severe occlusions, observed by five calibrated cameras. We follow the standard evaluation protocol.
    
    \item \textbf{Campus~\cite{belagiannis_3DPictorialStructures_2014}} captures multiple pedestrians in a courtyard with three calibrated cameras, testing robustness to fewer views and uncontrolled lighting.
    We follow the standard evaluation protocol.
\end{itemize}

\paragraph{Metrics}
For CMU Panoptic, we report Average Precision (AP) at multiple 3D distance thresholds (AP$_{25}$, AP$_{50}$, etc.), Recall at 500\,mm, and Mean Per-Joint Position Error (MPJPE).
For camera-layout generalization, we report the mean AP (mAP) averaged over the evaluated thresholds.
For Shelf and Campus, we report the standard Percentage of Correct Parts (PCP).
Metric details, including matching rules and threshold definitions, are provided in the supplementary material (\cref{app_sec:evaluation_metrics}).

\paragraph{Baselines}
We evaluate two variants of \textsc{\name}: the exact ILP solver (-ILP) and the BP relaxation (-BP).
We compare against fully supervised methods~\cite{lin_MultiViewMultiPerson3D_2021,wu_GraphBased3DMultiPerson_2021,choudhury_TEMPOEfficientMultiView_2023,chen_3DSAMultiview3D_2025}, recent self-supervised methods~\cite{srivastav_SelfPose3dSelfSupervisedMultiPerson_2024,liu_DSPDenseSparseParallel_2025}, and optimization-based methods~\cite{pirinen_DomesDronesSelfSupervised_2019,dong2019fast}.
For MvPose~\cite{dong2019fast}, the strongest optimization-based baseline, we run the official implementation using the same ViTPose++ 2D detections as \textsc{\name}.
This isolates the association and reconstruction stage from differences in 2D detector quality.
For other baselines, we report the numbers published by the respective authors.

\subsection{Quantitative Results}
\label{subsec:quantitative_results}
\vspace{-2mm}

\input{assets/tables/shelf_campus_results.tex}

\paragraph{CMU Panoptic}
\Cref{tab:panoptic_main} reports the main quantitative comparison on CMU Panoptic, where \textsc{\name} consistently improves over optimization-based baselines and recent self-supervised methods while remaining competitive with fully supervised approaches.

Against MvPose \cite{dong2019fast}, the strongest optimization-based baseline, both \name variants improve across metrics under identical 2D detections.
\name-ILP obtains the global optimum of the discrete objective over the retained set and achieves the lowest optimization-based MPJPE, reducing error from 26.46\,mm to 22.78\,mm.
\name-BP yields soft hyperedge marginals for ranking predictions and achieves strong AP scores; AP$_{25}$ improves from 37.63 for MvPose to 68.88.

Compared with fully supervised methods, \textsc{\name} remains competitive despite not training a 3D pose model.
Although direct 3D supervision benefits strict precision metrics, \textsc{\name} surpasses recent self-supervised approaches, including SelfPose3d~\cite{srivastav_SelfPose3dSelfSupervisedMultiPerson_2024} and DSP~\cite{liu_DSPDenseSparseParallel_2025}, on all reported metrics.
In particular, \name-ILP achieves lower MPJPE than DSP (22.78\,mm vs.\ 23.10\,mm), while \name-BP improves AP$_{25}$ by 13.75 points over SelfPose3d (55.13 vs. 68.88).

\input{assets/figures/qualitative_panoptic.tex}

\paragraph{Shelf and Campus}
\Cref{tab:results_shelf_campus} reports PCP results on Shelf and Campus.
On Shelf, \textsc{\name} achieves an average PCP of 96.2\%, outperforming the self-supervised baseline SelfPose3d (95.1\%) and remaining competitive with the optimization-based MvPose baseline (96.9\%).
The proposed method performs particularly well on Actor 1, reaching 99.8\% PCP.

On Campus, \name-ILP achieves 97.3\% avg. PCP, outperforming MvPose (96.3\%) and SelfPose3d (87.9\%).
This matches fully supervised TEMPO while requiring no 3D pose model training.

\subsection{Qualitative Results}
\label{subsec:qualitative_results}
\vspace{-2mm}

\input{assets/figures/qualitative_shelf.tex}

\paragraph{Comparison on CMU Panoptic}
\cref{fig:panoptic_qualitative} compares the reconstruction results of \name-ILP against MvPose \cite{dong2019fast} on the CMU Panoptic dataset. 
As illustrated in the zoomed-in regions, MvPose fails to establish correct correspondences for the highlighted individual, resulting in a missing reconstruction. 
In contrast, our method successfully processes the multi-view information and accurately reconstructs all individuals in the scene.

\paragraph{Shelf Ground-Truth Inaccuracies}
\Cref{fig:shelf_qualitative} compares our predictions with the public Shelf annotations.
We observe several frames in which the annotated 3D pose appears misaligned with the image evidence.
In the highlighted examples, the public annotations deviate from the visible actor location, whereas the \textsc{\name} prediction is visually consistent with the images.
Such cases can penalize quantitatively correct predictions under PCP.
For example, in the top example, the prediction receives 0\% PCP for the right lower and upper arms despite visually matching the actor; in the bottom example, the right lower arm, upper arm, and head are similarly penalized.

\subsection{Generalization to Camera Setups, Scalability, and Runtime}
\label{subsec:generalization_analysis}
\vspace{-2mm}

\input{assets/tables/generalization_cam_setup.tex}
\textbf{Generalization.}
\Cref{tab:generalization_cam_setup} evaluates generalization across Panoptic camera setups with varying camera numbers and placements, testing whether methods adapt to new layouts without retraining.
\textsc{\name} remains stable across setups and consistently outperforms optimization-based and self-supervised baselines.
In sparse 4-view setups such as CMU3, \textsc{\name} achieves 74.43 mAP, versus 59.74 mAP for MvPose~\cite{dong2019fast}.
The BP relaxation closely tracks \name-ILP across these settings, indicating that the scalable relaxation preserves similar generalization behavior.

\input{assets/figures/runtime_panoptic.tex}

\textbf{Scalability and Runtime.} 
\Cref{fig:sub_left} analyzes geometric pruning.
Although possible hyperedges increase with the number of views, only a small fraction satisfies the geometric consistency threshold, yielding a sparse candidate set for optimization.
\Cref{fig:sub_right} reports GH200 inference runtime as the number of cameras and pruning threshold $\tau$ vary.
MvPose scales efficiently via pairwise correspondences, whereas \name-ILP slows as the retained hyperedge set grows, especially for larger $V$ and $\tau$, reflecting branch-and-cut worst-case complexity.
In contrast, \name-BP uses batched GPU message passing and maintains an approx. constant runtime of 5\,ms while closely tracking ILP accuracy.
Runtimes exclude 2D pose detection and measure only association and reconstruction; candidate hyperedge-construction GPU memory is reported in the supplementaries (\cref{app_subsec:hierarchical_hypergraph_construction}).

%% file: assets/tables/panoptic_results.tex
\begin{table}[t!]
    \caption{
    \textbf{Quantitative comparison on the CMU Panoptic dataset \cite{panoptic}.}
    We report Average Precision (AP) at millimeter thresholds, Recall, and Mean Per Joint Position Error
    (MPJPE) in mm. 
    $^{\dagger}$ uses 9 temporal frames as input. 
    $^{\ddagger}$ uses the same 2D keypoint detector as our method. 
    Best results per supervision category (full-, self-, and optimization-based) are highlighted in {\bestFull blue}, {\bestSelf orange}, and {\bestOpt green}.
    }
    \label{tab:panoptic_main}
    \centering
    \small
    \setlength{\tabcolsep}{3pt}
    \begin{tabular*}{\textwidth}{@{\extracolsep{\fill}} @{\hspace{2pt}}l
    S[table-format=2.2] S[table-format=2.2] S[table-format=2.2]
    S[table-format=3.2] S[table-format=2.2] S[table-format=3.2] }
        \toprule                                                                                   %
        \multirow{2}{*}{\textbf{Method}}                                           & \multicolumn{4}{c}{\textbf{Average Precision (AP)} ($\uparrow$)} & {\textbf{Recall} ($\uparrow$)} & {\textbf{Error} ($\downarrow$)} \\
        \cmidrule(lr){2-5} \cmidrule(lr){6-6} \cmidrule(lr){7-7}                           %
                                                                                   & {25}                             & {50}                           & {100}                          & {150}              & {@500}             & {MPJPE}            \\
        \midrule                                                                                   %
        \multicolumn{7}{l}{\textit{Fully-Supervised}}                                      \\
        \hspace{1em} Plane Sweep Pose \cite{lin_MultiViewMultiPerson3D_2021}               & 92.12                            & 98.96                          & \bestFull 99.81                & 99.84              & {--}               & 16.75              \\
        \hspace{1em} Wu \textit{et al.} \cite{wu_GraphBased3DMultiPerson_2021}             & 93.93                            & 98.93                          & 99.78                          & 99.90              & \bestFull 99.97    & 15.63              \\
        \hspace{1em} TEMPO \cite{choudhury_TEMPOEfficientMultiView_2023}                   & 89.01                            & \bestFull 99.08                & 99.76                          & \bestFull 99.93    & {--}               & 14.68              \\
        \hspace{1em} VoxelPose + 3DSA \cite{chen_3DSAMultiview3D_2025}                     & \bestFull 94.20                  & 98.49                          & 99.21                          & 99.31              & {--}               & \bestFull 13.98    \\
        \midrule                                                                                   %
        \multicolumn{7}{l}{\textit{Self-Supervised}}                                       \\
        \hspace{1em} SelfPose3d \cite{srivastav_SelfPose3dSelfSupervisedMultiPerson_2024}  & 55.13                            & \bestSelf 96.44                & \bestSelf 98.46                & \bestSelf 98.98    & \bestSelf 99.60    & 24.47              \\
        \hspace{1em} DSP$^{\dagger}$ \cite{liu_DSPDenseSparseParallel_2025}                & \bestSelf 57.60                  & 86.10                          & 94.00                          & {--}               & {--}               & \bestSelf 23.10    \\
        \midrule                                                                                   %
        \multicolumn{7}{l}{\textit{Optimization-Based}}                                    \\
        \hspace{1em} ACTOR \cite{pirinen_DomesDronesSelfSupervised_2019}                   & {--}                             & {--}                           & {--}                           & {--}               & {--}               & 168.40             \\
        \hspace{1em} MvPose$^{\ddagger}$ \cite{dong2019fast}                               & 37.63                            & 95.70                          & 97.84                          & 98.28              & 99.60              & 26.46              \\
        \hspace{1em} \name-ILP (Ours)                                                      & 66.70                            & 98.23                          & \bestOpt 99.43                 & \bestOpt 99.62     & \bestOpt 99.81     & \bestOpt 22.78     \\
        \hspace{1em} \name-BP (Ours)                                                       & \bestOpt 68.88                   & \bestOpt 98.37                 & 99.42                          & 99.61              & \bestOpt 99.81              & \bestOpt 22.78              \\
        \bottomrule
    \end{tabular*}
    \vspace{-5.0mm}
\end{table}

%% file: assets/tables/shelf_campus_results.tex
\begin{table}[t!]
    \centering
    \caption[Quantitative Comparison Shelf and Campus]{
    \textbf{Quantitative comparison on the Shelf \cite{belagiannis_3DPictorialStructures_2014} and Campus \cite{belagiannis_3DPictorialStructures_2014} datasets}. 
    Best results per supervision category (full-, self-, and optimization-based) are highlighted in {\bestFull{blue}}, {\bestSelf{orange}}, and {\bestOpt{green}}. 
    A1, A2, and A3 refer to Actor 1, 2, and 3, respectively.}
    \label{tab:results_shelf_campus}
    \small
    \setlength{\tabcolsep}{3pt}
    \begin{tabular*}{\textwidth}{ @{\extracolsep{\fill}} l *{8}{S[table-format=2.1]} @{} }
        \toprule                                                                                              & \multicolumn{4}{c}{\textbf{Shelf (PCP \%)} ($\uparrow$)} & \multicolumn{4}{c}{\textbf{Campus (PCP \%)} ($\uparrow$)} \\
        \cmidrule(lr){2-5} \cmidrule(l){6-9} 
        \textbf{Method}                                                                                       & {A1}            & {A2}            & {A3}            & {Avg.}          & {A1}            & {A2}            & {A3}            & {Avg.}          \\
        \midrule                                                                                              
        \multicolumn{9}{l}{\textit{Fully Supervised}}                            \\
        \hspace{1em} VoxelPose \cite{tu_VoxelPoseMulticamera3D_2020}            & {\bestFull{99.3}} & 94.1            & 97.6            & 97.0            & 97.6            & 93.8            & {\bestFull{98.8}} & 96.7            \\
        \hspace{1em} Wu et al. \cite{wu_GraphBased3DMultiPerson_2021}           & {\bestFull{99.3}} & {\bestFull{96.5}} & 97.3            & {\bestFull{97.7}} & {--}            & {--}            & {--}            & {--}            \\
        \hspace{1em} TEMPO \cite{choudhury_TEMPOEfficientMultiView_2023}        & {\bestFull{99.3}} & 95.1            & {\bestFull{97.8}} & 97.4            & {\bestFull{97.7}} & {\bestFull{95.5}} & 97.9            & {\bestFull{97.3}} \\
        \midrule                                                                                              
        \multicolumn{9}{l}{\textit{Self-Supervised}}                             \\
        \hspace{1em} SelfPose3d                                                 & {\bestSelf{97.2}} & {\bestSelf{90.3}} & {\bestSelf{97.9}} & {\bestSelf{95.1}} & {\bestSelf{92.5}} & {\bestSelf{82.2}} & {\bestSelf{89.2}} & {\bestSelf{87.9}} \\
        \midrule                                                                                              
        \multicolumn{9}{l}{\textit{Optimization-Based}}                          \\
        \hspace{1em} 3DPS \cite{belagiannis_3DPictorialStructures_2014}         & 75.3            & 69.7            & 87.6            & 77.5            & 93.5            & 75.7            & 84.4            & 84.5            \\
        \hspace{1em} MvPose \cite{dong2019fast}              & 98.8            & {\bestOpt{94.1}}  & {\bestOpt{97.8}}  & {\bestOpt{96.9}}  & 97.6            & 93.3            & 98.0            & 96.3            \\
        \hspace{1em} \name-ILP (Ours)                                           & {\bestOpt{99.8}}  & 92.4            & 96.3            & 96.2            & {\bestOpt{99.4}}  & {\bestOpt{94.3}}  & {\bestOpt{98.1}}  & {\bestOpt{97.3}}  \\
        \hspace{1em} \name-BP (Ours)                                            & {\bestOpt{99.8}}  & 92.4            & 96.3            & 96.2            & {\bestOpt{99.4}}  & {\bestOpt{94.3}}  & 93.6  & 95.7  \\
        \bottomrule
    \end{tabular*}
    \vspace{-5mm}
\end{table}

%% file: assets/figures/qualitative_panoptic.tex
\begin{figure}[t!]
    \centering
    \includegraphics[width=\textwidth]{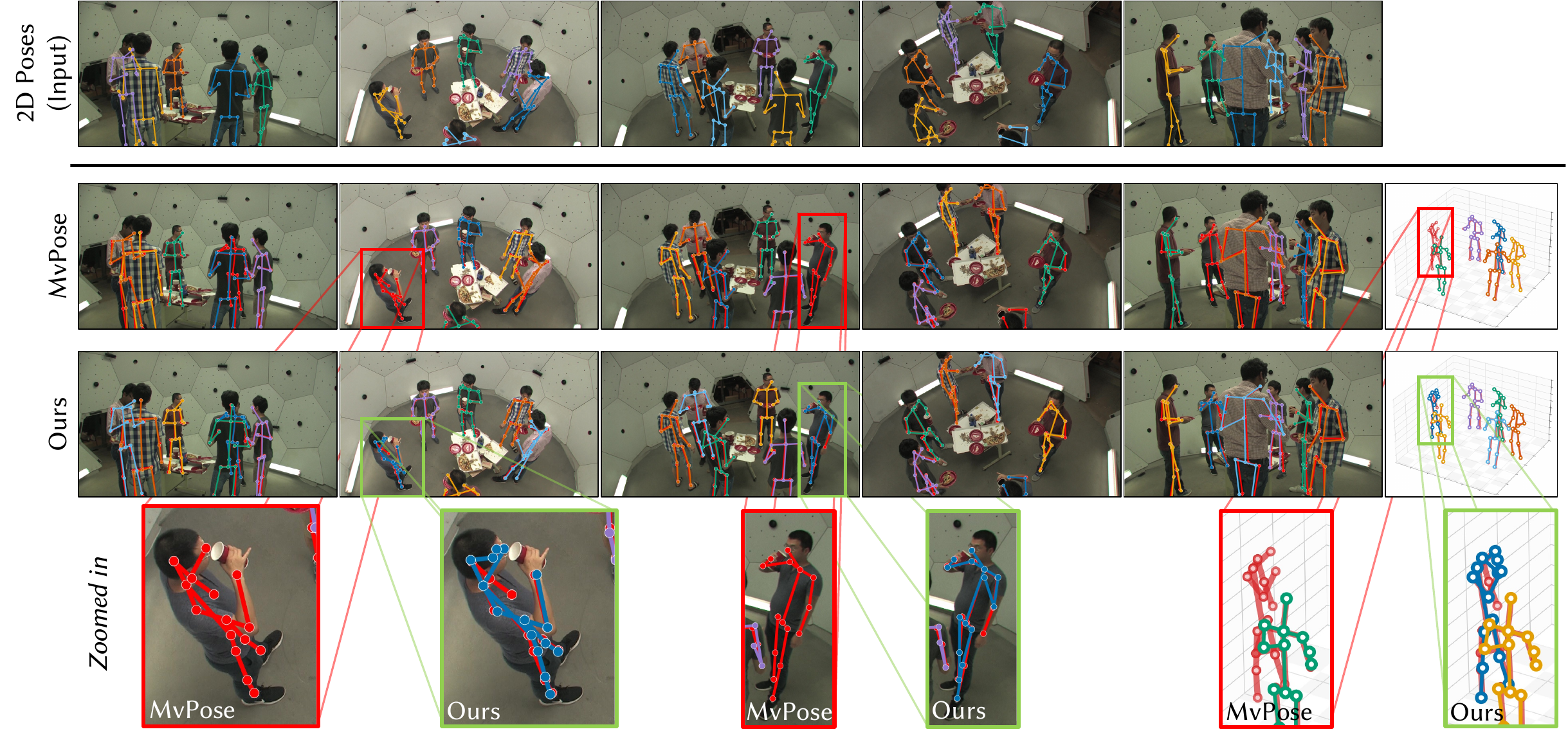}
    \caption{
    \textbf{COMPOSE recovers people missed by pairwise association}.
    Shown are Panoptic \cite{panoptic} 2D detections from ViTPose++ \cite{xu2022vitpose}, and 3D predictions of MvPose \cite{dong2019fast} versus Ours. 
    3D GT is red; predictions are colored.
    MvPose misses the highlighted person; \name reconstructs everyone.
    }
    \label{fig:panoptic_qualitative}
    \vspace{-3mm}
\end{figure}

%% file: assets/figures/qualitative_shelf.tex
\begin{figure}[t!]
    \centering
    \includegraphics[width=\textwidth]{
        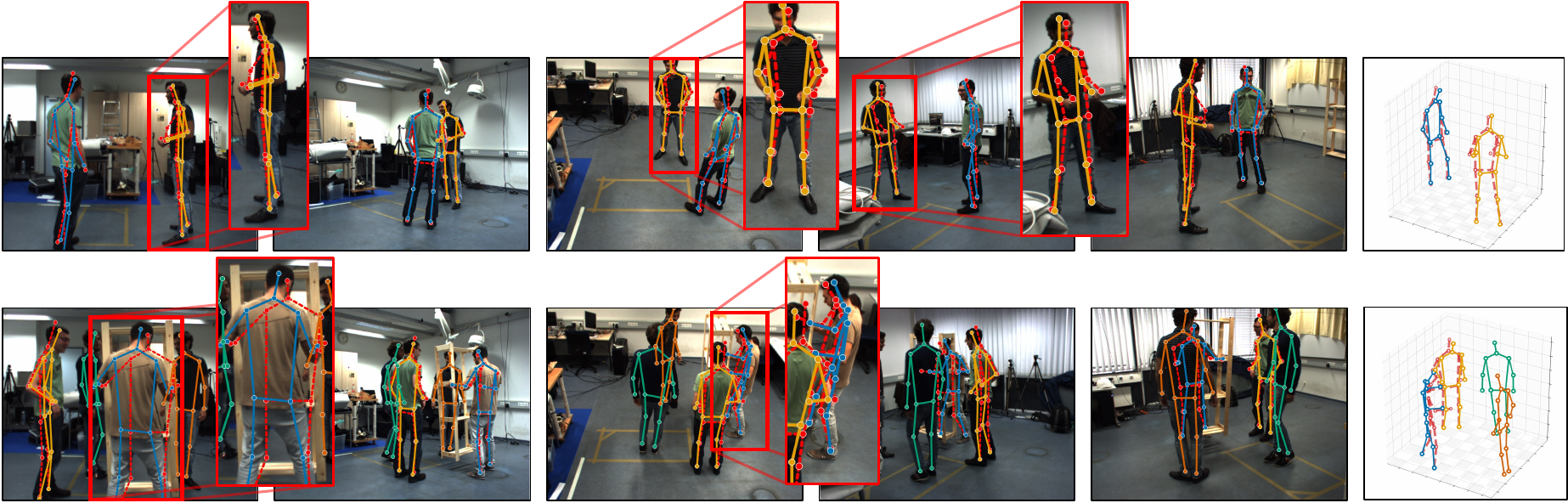
    }
    \caption{
    \textbf{Visual evidence exposes annotation noise}.
    In occlusion cases \cite{belagiannis_3DPictorialStructures_2014}, \name aligns with image evidence even when the 3D annotation (dashed red) penalizes it (solid colored).
    }
    \label{fig:shelf_qualitative}
    \vspace{-4mm}
\end{figure}

%% file: assets/tables/generalization_cam_setup.tex
\begin{table}[t!]
    \caption{\textbf{Generalization across diverse camera setups on CMU Panoptic \cite{panoptic}.}
    The number of available cameras for each setup is indicated in parentheses.
    Best results are \textbf{highlighted}}
    \label{tab:generalization_cam_setup}
    \centering
    \small
    \setlength{\tabcolsep}{3pt}
    \begin{tabular*}{\textwidth}{ @{\extracolsep{\fill}} l l *{8}{S[table-format=2.2]} }
        \toprule                                                                                                                                                                                                                                                     %
        \multirow{2}{*}{\textbf{Type}} & \multirow{2}{*}{\textbf{Method}}                  & \multicolumn{2}{c}{\textbf{CMU1 (7)}} & \multicolumn{2}{c}{\textbf{CMU2 (7)}} & \multicolumn{2}{c}{\textbf{CMU3 (4)}} & \multicolumn{2}{c}{\textbf{CMU4 (4)}} \\
        \cmidrule(lr){3-4} \cmidrule(lr){5-6} \cmidrule(lr){7-8} \cmidrule(lr){9-10}                                                                                                                                                                                 %
                                       &                                                   & {mAP}            & {Rec.$_{500}$}            & {mAP}            & {Rec.$_{500}$}            & {mAP}            & {Rec.$_{500}$}            & {mAP}            & {Rec.$_{500}$}            \\
        \midrule                                                                                                                                                                                                                                                     %
        Self-Sup.                      & SelfPose3d \cite{srivastav_SelfPose3dSelfSupervisedMultiPerson_2024}  &  74.50           &  97.98           & 59.06            & 94.32            & 61.43            & 83.96            & 62.85            & 98.32            \\
        \midrule                                                                                                                                                                                                                                                     %
        \multirow{3}{*}{Optim.}        & MvPose \cite{dong2019fast}                        & 84.62            & 99.53            & 80.07            & 99.37            & 59.74            & \textbf{98.80}   & 74.85            & 98.59            \\
                                       & \name-ILP (Ours)                                  & \textbf{88.49}   & \textbf{99.61}   & \textbf{84.45}   & \textbf{99.58}   & 73.83            & 98.40            & \textbf{80.17}   & \textbf{99.31}   \\
                                       & \name-BP (Ours)                                   & 88.10            & 99.45            & 84.34            & 99.41            & \textbf{74.43}   & 98.39            & 79.60            & \textbf{99.31}   \\
        \bottomrule
    \end{tabular*}
    \vspace{-5mm}
\end{table}

%% file: assets/figures/runtime_panoptic.tex
\begin{figure}[t!]
    \centering

    \subfloat[Hyperedge pruning vs. camera view count.]{
        \includegraphics[width=0.48\textwidth]{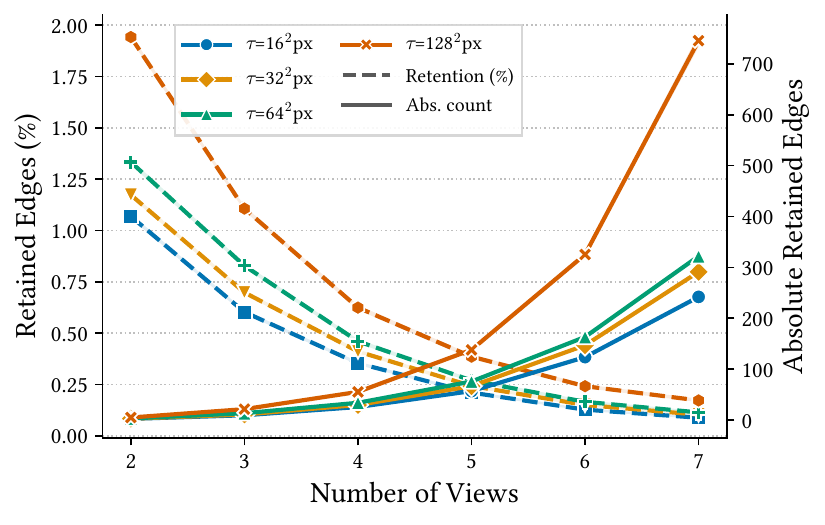}
        \label{fig:sub_left}
    }
    \hfill
    \subfloat[Runtime vs. camera view count.]{
        \includegraphics[width=0.48\textwidth]{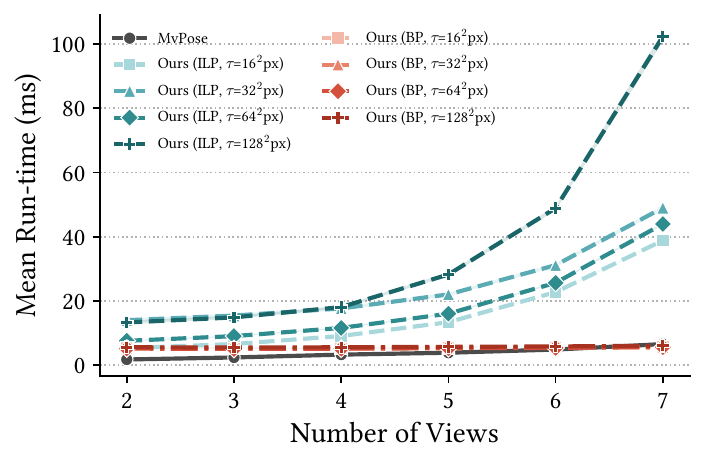}
        \label{fig:sub_right}
    }

    \caption{\textbf{Scalability analysis on Panoptic \cite{panoptic} across varying $\tau$ thresholds.}
    (a) Absolute count and percentage of hyperedges pruned to maintain optimization efficiency.
    (b) Runtime comparison of MvPose \cite{dong2019fast}, \name-ILP, and \name-BP as the number of available cameras increases.}
    \label{fig:runtime_analysis}
    \vspace{-5.0mm}
\end{figure}

%% file: corpus/6_conclusion.tex
\section{Concluding Remarks}
\label{sec:conclusion}
\vspace{-2mm}

We present \textsc{\name}, a higher-order combinatorial formulation for multi-view 3D pose estimation, solved at test time without 3D supervision.
By recasting correspondence search as weighted exact-cover over a hypergraph of person hypotheses, our approach enforces multi-view consistency holistically rather than reconciling pairwise matches post-hoc.
We provide two complementary solvers: an Integer Linear Programming formulation and a scalable Belief Propagation relaxation with parallelizable GPU inference and soft association beliefs.
Extensive experiments across three benchmarks show state-of-the-art performance among optimization-based and self-supervised methods, while remaining competitive with fully supervised approaches.
Our proposed method highlights the potential of higher-order graph formulations for multi-view 3D pose estimation, with applications to tracking \cite{wang2025trackor} and team-behavior analysis \cite{wang2025beyond}.
We release the code to facilitate further research.

\noindent\textbf{Limitations and Future Work.} 
\textsc{\name} requires calibrated cameras for triangulation, and cannot reconstruct persons visible in only a single view as there are no other observations to triangulate against. 
The formulation still operates frame-by-frame and does not exploit the temporal information readily available in multi-view recordings. Future work will explore weakly calibrated regimes, monocular priors for single-view recovery, and temporal hyperedges to enable multi-person tracking.

%% file: corpus/99_epilogue.tex
\section*{Acknowledgments}
The authors are grateful for support from the UK AI Research Resource (AIRR) through grant 0251-4584-0945-1 and from the Excellence Strategy of local and state governments in Bavaria, Germany, as well as computational resources of the LRZ AI service infrastructure provided by the Leibniz Supercomputing Center (LRZ), the German Federal Ministry of Education and Research (BMBF), and the Bavarian State Ministry of Science and the Arts (StMWK).
T. B. was supported by the UKRI Engineering and Physical Sciences Research Council (EPSRC) through the Future Leaders Fellowship [grant number MR/Y018818/1].
L.B. was supported by the UK Royal Society through grant NIF/R1/254128.

%% file: supplementary/0_main.tex
\begin{center}
    {\Large\bfseries Supplementary Material}
\end{center}

In this supplementary material, we collect the additional context, derivations, implementation details, and empirical evidence supporting the main manuscript.
We begin with related work in \cref{sec:related_works}, then expand on the ILP--MAP connection and BP relaxation in \cref{app_sec:additional_proofs}; the following sections detail implementation and evaluation protocols (\cref{app_sec:additional_methods}), report extended experiments and ablations (\cref{app_sec:additional_experiments}), and discuss broader societal impact (\cref{app_sec:societal_impact}).

\input{supplementary/1_related_works}
\input{supplementary/1_additional_proofs}

\input{supplementary/2_additional_methods}
\input{supplementary/3_additional_experiments}
\input{supplementary/4_epilogue}

%% file: supplementary/1_related_works.tex
\section{Related Works}
\label{sec:related_works}

Multi-view multi-human 3D pose estimation can be stratified into two distinct paradigms: optimization-based and learning-based approaches \cite{nogueira_MarkerlessMultiview3D_2025}.
Optimization-based methods rely on geometric triangulation and solver-based consensus to lift 2D priors (e.g., from an off-the-shelf 2D pose detector) into 3D space \cite{belagiannis_3DPictorialStructures_2014}.
Conversely, learning-based approaches use deep neural networks to regress 3D poses from visual features, typically requiring annotated datasets \cite{tu_VoxelPoseMulticamera3D_2020}.

\paragraph{Optimization-Based Approaches}
As a pioneering work, Belagiannis et al. introduced \textit{3D Pictorial Structures}, a conditional random field framework that uses a discrete state space and multi-view potential functions to resolve identity and body-part ambiguities across multiple camera views \cite{belagiannis_3DPictorialStructures_2014}. 
Dong et al. \cite{dong2019fast} establish multi-view correspondences by formulating the multi-way matching problem as a convex objective that simultaneously clusters 2D detections across all views using appearance similarity and geometric consistency cues while enforcing a \textit{cycle-consistency} constraint. 
To address occlusions and crowded scenes, Zhou et al. \cite{zhou_QuickPoseRealtimeMultiview_2022} introduce multi-view association at the level of partial skeleton proposals instead of body-level. 
Zhang et al. \cite{zhang_4DAssociationGraph_2020} introduce the temporal dimension to the task and propose a spatio-temporal graph formulation for both spatial and temporal associations. 
These optimization-based methods are efficient and require minimal computational resources. 
However, compared to learnable methods, they struggle with noisy 2D detections and occlusions.

\paragraph{Learning-Based Approaches}
The prevailing body of work leverages deep neural networks to regress 3D human poses. 
Early methods encode the environment as 3D voxel grids, learning with 3D CNNs \cite{iskakov_LearnableTriangulationHuman_2019,tu_VoxelPoseMulticamera3D_2020,chen_3DSAMultiview3D_2025,srivastav_SelfPose3dSelfSupervisedMultiPerson_2024}. 
However, these are computationally expensive due to cubic complexity and often overfit to specific camera setups \cite{liao_MultipleViewGeometry_2024}. 
To address this, recent methods project 3D hypotheses onto 2D image planes to leverage 2D features, improving both speed and flexibility \cite{wang_DirectMultiviewMultiperson_2021,liao_MultipleViewGeometry_2024,chharia_MVSSMMultiViewState_2025}. 
Despite these architectural advances, fully supervised models fundamentally depend on scarce, labor-intensive 3D-annotated data. 
While self-supervised approaches aim to reduce reliance on labeled data, they currently underperform and struggle to generalize to novel environments \cite{srivastav_SelfPose3dSelfSupervisedMultiPerson_2024,liu_DSPDenseSparseParallel_2025}. 
Thus, optimization-based approaches remain a critical alternative, offering robust generalization in diverse settings without requiring costly 3D supervision.

\paragraph{Multi-View Synchronization}
Reconciling locally consistent pairwise associations into a globally coherent structure has historically been treated as a \emph{synchronization problem}, which is defined as the recovery of absolute quantities from a collection of ratios \cite{govindu2014averaging,govindu2004lie,arrigoni2017synchronization}.
Pioneering works use Lie group theory to average rigid motions for \textit{Structure-from-Motion}, extending to SE(3) via spectral decomposition and semidefinite programming, offering closed-form solutions and strong duality guarantees \cite{arrigoni2016spectral}. 
Recent advancements explore probabilistic synchronization \cite{birdal2019probabilistic}, quantum permutation synchronization for non-convex optimization \cite{birdal2021quantum}, and generative synchronization to align multiple joint diffusion processes \cite{lee2023syncdiffusion,kim2024synctweedies}. 
While cycle consistency can be mathematically enforced across pairwise matches, relying solely on synchronizing dyadic relations can propagate errors when detections in individual views are noisy or occluded.

%% file: supplementary/1_additional_proofs.tex
\section{Background}
\label{app_sec:additional_proofs}

We provide further intuition regarding the connection between integer linear programming (ILP), the Gibbs distribution, probabilistic relaxations such as Markov random fields, and the resulting belief propagation algorithm.

\subsection{Integer Linear Programming (ILP)}

Let us first rephrase the ILP in \cref{def:ilp} as a minimization problem:

\begin{align}
    \label{app_eq:ilp}
    \min_{\bx}\quad - &\sum_{e \in \cE}(s(e) - \gamma)\, x_{e} \\
    \text{s.t.}\quad    &\sum_{e:\, u \in e}x_{e}= 1, \quad \forall\, u \in \cU \qquad \text{and} \qquad
                          x_{e}\in \{0, 1\}, \quad \forall\, e \in \cE
\end{align}

\noindent
This can be interpreted as minimizing a corresponding energy function:

\begin{align}
    \label{eq:energy_formulation}
    E(\bx) = \sum_{e \in \cE} c_e x_e, \qquad c_e := \gamma - s(e)
\end{align}

\noindent
where $c_e$ encodes the cost of selecting hyperedge $e$: a high-confidence hyperedge (large $s(e)$) yields a lower cost, making its selection energetically more favorable; the converse is true for low-confidence hyperedges.
The key to relaxing the objective is that any optimization problem of the form $\min_\bx E(\bx)$ can be cast as finding the mode of a corresponding probability distribution, the Gibbs distribution \cite{geman1988stochastic}.

\subsection{Gibbs Energy Formulation}

Given our energy function in \cref{eq:energy_formulation}, we now define the Gibbs distribution (also called the Boltzmann distribution \cite{murphy2012machine}) as:

\begin{equation}
    \label{eq:gibbs_supp}
    p_\beta(\bx) = \frac{1}{Z(\beta)} \exp\bigl(-\beta\, E(\bx)\bigr),
\end{equation}

\noindent
where $\beta > 0$ is an \emph{inverse temperature} parameter and $Z(\beta) = \sum_{\bx} \exp(-\beta\, E(\bx))$ is a \emph{partition function} to normalize the probability distribution \cite{koller2009probabilistic}.

The inverse temperature $\beta$ controls the concentration of the distribution around low-energy states.
In the \emph{high-temperature regime} ($\beta \to 0$), the distribution becomes uniform over all configurations, assigning equal probability regardless of energy.
In the \emph{low-temperature regime} ($\beta \to \infty$), the distribution concentrates its mass on the global energy minimizer(s):
\begin{equation}
    \label{eq:gibbs_limit}
    \lim_{\beta \to \infty} p_\beta(\bx) = 
    \begin{cases} 
        \frac{1}{|\mathcal{X}^*|} & \text{if } \bx \in \mathcal{X}^*, \\ 
        0 & \text{otherwise},
    \end{cases}
\end{equation}

\noindent
where $\mathcal{X}^* = \arg\min_{\bx} E(\bx)$ is the set of optimal configurations \cite{koller2009probabilistic}.
Consequently, for any $\beta > 0$, the \emph{maximum a posteriori} (MAP) estimate $\arg\max_{\bx}\, p_\beta(\bx)$ coincides with $\arg\min_{\bx}\, E(\bx)$, since $\beta$ is a positive scalar.
This relationship guarantees that solving the ILP is equivalent to MAP inference under the corresponding Gibbs distribution.

\input{assets/figures/factor_graph}

\begin{proposition}[ILP--MAP Equivalence]
    \label{prop:ilp_map_supp}
    Let $E(\bx) = \sum_{e \in \cE} (\gamma - s(e))\, x_e$ be the energy function corresponding to the negated ILP objective, and let $\mathcal{F} = \{ \bx \in \{0,1\}^{|\cE|} \mid \sum_{e:\, u \in e} x_e = 1,\; \forall\, u \in \cU \}$ denote the feasible set.
    Define the constrained Gibbs distribution:
    \begin{equation}
        \label{eq:constrained_gibbs}
        p_\beta(\bx) = \frac{1}{Z(\beta)} \exp\bigl(-\beta\, E(\bx)\bigr) \cdot \mathds{1}[\bx \in \mathcal{F}],
    \end{equation}
    where $\mathds{1}[\cdot]$ is the indicator function.
    Then, for any $\beta > 0$:
    \begin{equation}
        \argmax_{\bx}\; p_\beta(\bx) = \argmin_{\bx \in \mathcal{F}}\; \sum_{e \in \cE} (\gamma - s(e))\, x_e = \argmax_{\bx \in \mathcal{F}}\; \sum_{e \in \cE} (s(e) - \gamma)\, x_e.
    \end{equation}
    That is, MAP inference under $p_\beta$ recovers the optimal solution of the ILP in \cref{def:ilp}.
\end{proposition}

\begin{proof}
    The indicator $\mathds{1}[\bx \in \mathcal{F}]$ restricts the support to feasible configurations, and since $\log$ is strictly monotone and $\beta > 0$ is a positive constant, we have $\argmax_{\bx}\, p_\beta(\bx) = \argmin_{\bx \in \mathcal{F}}\, E(\bx) = \argmax_{\bx \in \mathcal{F}}\, \sum_{e \in \cE}(s(e) - \gamma)\, x_e$, recovering the ILP in \cref{def:ilp}.
    This classical equivalence underlies a broad family of methods bridging combinatorial optimization and probabilistic inference \cite{kirkpatrick1983optimization,geman1984stochastic}.
\end{proof}

\subsection{Markov Random Fields and Factor Graphs}
\label{app_subsec:mrfs}

We briefly review the graphical model formalism underlying our belief propagation solver, deferring to the seminal works of Pearl \cite{pearl2014probabilistic}, Yedidia et al.\ \cite{yedidia2001bethe,yedidia2003understanding}, and Wainwright \& Jordan \cite{wainwright2008graphical} for a thorough treatment.

A \emph{Markov random field} (MRF) is an undirected graphical model in which a joint distribution factorizes over the cliques of a graph \cite{pearl2014probabilistic}:
\begin{equation}
    \label{eq:mrf}
    p(\bx) = \frac{1}{Z} \prod_{c \in \mathcal{C}} \psi_c(\bx_c),
\end{equation}
where $\psi_c$ are non-negative potential functions and $Z$ is the partition function.
This factorization can be made explicit through a \emph{factor graph} \cite{kschischang2001factor}, a bipartite graph of variable nodes and factor nodes, where each factor $\psi_c$ connects to exactly those variables it depends on.
Factor graphs provide the natural domain for message-passing algorithms such as belief propagation \cite{yedidia2003understanding}.

Our Gibbs distribution in \cref{eq:constrained_gibbs} admits precisely this structure.
Specifically, it factorizes into \emph{unary factors} $f_e(x_e)$, encoding the plausibility of each hyperedge, and \emph{constraint factors} $g_u(\bx_{\cE(u)})$, enforcing that each detection is explained exactly once:
\begin{equation}
    \label{eq:gibbs_factored}
    p_\beta(\bx) = \frac{1}{Z(\beta)} \prod_{e \in \cE} f_e(x_e) \prod_{u \in \cU} g_u\!\bigl(\bx_{\cE(u)}\bigr).
\end{equation}
Two hyperedge variables are coupled if and only if they share a detection node, making the graph sparse.
Moreover, this sparsity is further amplified by geometric pruning, which eliminates most candidate hyperedges, leaving only a few variable pairs to interact through shared detections.
\Cref{fig:factor_graph} illustrates this construction on a small example with three views, two detections per view, and three candidate hyperedges.
Each hyperedge $e$ becomes a binary variable node $x_e$, each detection $u$ induces a constraint factor $g_u$, and unary factors $f_e$ sit atop each variable.
Notably, $x_{e_1}$ and $x_{e_3}$ are coupled through $g_{u_1^1}$ because both hypotheses claim the same detection.
The constraint factor enforces that at most one can be selected.

\paragraph{Unary (plausibility) factors}
To encode the ILP objective in the Gibbs framework, we define each unary factor as:
\begin{equation}
\label{app_eq:unary}
  f_e(x_e) := \exp\bigl(\beta\,(s(e)-\gamma)\,x_e\bigr),
\end{equation}
so that $f_e(0) = 1$ and $f_e(1) = \exp(\beta(s(e)-\gamma))$.
The resulting log-potential ratio is:
\begin{equation}
\label{eq:logpot}
  \phi_e \coloneqq \log\frac{f_e(1)}{f_e(0)} = \beta\,(s(e)-\gamma),
\end{equation}
which directly encodes the ILP coefficient scaled by the inverse temperature $\beta$.

\paragraph{Constraint factors}
The ILP feasibility constraints are enforced by:
\begin{equation}\label{app_eq:exclusion_exact}
  g_u\bigl(\bx_{\cE(u)}\bigr)
  =
  \begin{cases}
    1, & \sum_{e \in \cE(u)} x_e = 1,\\[2pt]
    0, & \text{otherwise},
  \end{cases}
\end{equation}
which assigns zero probability to any configuration in which a detection is left uncovered or multiply assigned.
Together with \cref{app_eq:unary}, MAP inference under \cref{eq:gibbs_factored} recovers the ILP solution, as established in \cref{prop:ilp_map_supp}.

\subsection{Constraint Relaxation for Belief Propagation (BP)}
\label{app_subsec:relaxation}

The hard equality constraint in \cref{app_eq:exclusion_exact} assigns zero probability to any configuration where a detection is uncovered ($\sum_{e \in \cE(u)} x_e = 0$) or multiply assigned.
To effectively apply loopy belief propagation, we must relax these to soft penalties, avoiding $\log 0$ and resulting numerical instabilities \cite{wainwright2008graphical}.

To obtain a well-behaved message-passing scheme, we relax the constraint factor to the ``at-most-one'' form with a tunable penalty $\eta \geq 0$ for uncovered detections, as introduced in the main paper (\cref{eq:exclusion}):
\begin{equation}\label{eq:exclusion_relaxed_supp}
  g_u^{(\eta)}\!\bigl(\bx_{\cE(u)}\bigr)
  =
  \mathds{1}\!\left[\sum_{e\in\cE(u)}x_e\le 1\right]
  \exp\!\left(
    -\eta\;\mathds{1}\!\left[\sum_{e\in\cE(u)}x_e=0\right]
  \right).
\end{equation}
This relaxation replaces the hard zero with a smooth penalty: uncovered detections are discouraged but not forbidden, providing a continuous landscape that BP can navigate.
Setting $\eta = 0$ yields the pure exclusion constraint used in the BP updates derived in the main paper, while $\eta \to \infty$ recovers the exact coverage constraint in \cref{app_eq:exclusion_exact}.
We include \cref{fig:bp_messages} for a conceptual illustration of how BP works on the factor graph.

\input{assets/figures/belief_propagation}

%% file: assets/figures/factor_graph.tex
\begin{figure}[t!]
    \centering
    \resizebox{\textwidth}{!}{%
    \begin{tikzpicture}[
        varnode/.style={circle, draw=black, fill=blue!15, minimum size=20pt, inner sep=0pt, font=\small},
        constfactor/.style={rectangle, draw=black, fill=red!25, minimum size=14pt, inner sep=2pt, font=\scriptsize},
        unaryfactor/.style={rectangle, draw=black, fill=green!20, minimum size=11pt, inner sep=1pt, font=\scriptsize},
        detection/.style={circle, draw, fill=blue!40, minimum size=15pt, inner sep=0pt, font=\scriptsize},
        fedge/.style={thick, draw=black!70},
        viewlabel/.style={font=\footnotesize\bfseries, text=gray},
    ]

    \node[viewlabel] at (-4.5, 3.8) {Hypergraph};

    \node[viewlabel, text=blue!70!black] at (-6.2, 2.8) {View 1};
    \node[viewlabel, text=blue!70!black] at (-4.5, 2.8) {View 2};
    \node[viewlabel, text=blue!70!black] at (-2.8, 2.8) {View 3};

    \node[detection] (u11) at (-6.2, 1.8) {$u_1^1$};
    \node[detection] (u21) at (-6.2, 0.4) {$u_2^1$};
    \node[detection] (u12) at (-4.5, 1.8) {$u_1^2$};
    \node[detection] (u22) at (-4.5, 0.4) {$u_2^2$};
    \node[detection] (u13) at (-2.8, 1.8) {$u_1^3$};
    \node[detection] (u23) at (-2.8, 0.4) {$u_2^3$};

    \begin{scope}[on background layer]
        \fill[orange!20, rounded corners=8pt] 
            ($(u11) + (-0.35, 0.30)$) -- 
            ($(u13) + (0.35, 0.30)$) -- 
            ($(u13) + (0.35, -0.30)$) -- 
            ($(u11) + (-0.35, -0.30)$) -- cycle;
    \end{scope}
    \node[font=\scriptsize, text=orange!70!black] at (-4.5, 2.4) {$e_1$};

    \begin{scope}[on background layer]
        \fill[purple!15, rounded corners=8pt] 
            ($(u21) + (-0.35, 0.30)$) -- 
            ($(u23) + (0.35, 0.30)$) -- 
            ($(u23) + (0.35, -0.30)$) -- 
            ($(u21) + (-0.35, -0.30)$) -- cycle;
    \end{scope}
    \node[font=\scriptsize, text=purple!70!black] at (-4.5, -0.2) {$e_2$};

    \begin{scope}[on background layer]
        \draw[teal!50, thick, rounded corners=4pt, fill=teal!8] 
            ($(u11) + (-0.4, 0.4)$) -- 
            ($(u11) + (0.4, 0.4)$) --
            ($(u22) + (0.4, 0.3)$) -- 
            ($(u22) + (0.4, -0.3)$) -- 
            ($(u22) + (-0.4, -0.3)$) --
            ($(u11) + (-0.4, -0.3)$) -- cycle;
    \end{scope}
    \node[font=\scriptsize, text=teal!70!black] at (-5.2, 1.0) {$e_3$};

    \draw[->, ultra thick, gray] (-1.6, 1.1) -- (-0.4, 1.1);

    \node[viewlabel] at (3.0, 3.8) {Factor Graph};

    \node[unaryfactor] (fe1) at (1.5, 3.2) {$f_{e_1}$};
    \node[unaryfactor] (fe2) at (3.0, 3.2) {$f_{e_2}$};
    \node[unaryfactor] (fe3) at (4.5, 3.2) {$f_{e_3}$};

    \node[varnode] (xe1) at (1.5, 2.0) {$x_{e_1}$};
    \node[varnode] (xe2) at (3.0, 2.0) {$x_{e_2}$};
    \node[varnode] (xe3) at (4.5, 2.0) {$x_{e_3}$};

    \draw[fedge] (fe1) -- (xe1);
    \draw[fedge] (fe2) -- (xe2);
    \draw[fedge] (fe3) -- (xe3);

    \node[constfactor] (gu11) at (0.5, 0.4) {$g_{u_1^1}$};
    \node[constfactor] (gu21) at (1.4, 0.4) {$g_{u_2^1}$};
    \node[constfactor] (gu12) at (2.3, 0.4) {$g_{u_1^2}$};
    \node[constfactor] (gu22) at (3.2, 0.4) {$g_{u_2^2}$};
    \node[constfactor] (gu13) at (4.1, 0.4) {$g_{u_1^3}$};
    \node[constfactor] (gu23) at (5.0, 0.4) {$g_{u_2^3}$};

    \draw[fedge, orange!70!black] (xe1) -- (gu11);
    \draw[fedge, orange!70!black] (xe1) -- (gu12);
    \draw[fedge, orange!70!black] (xe1) -- (gu13);

    \draw[fedge, purple!60!black] (xe2) -- (gu21);
    \draw[fedge, purple!60!black] (xe2) -- (gu22);
    \draw[fedge, purple!60!black] (xe2) -- (gu23);

    \draw[fedge, teal!70!black] (xe3) -- (gu11);
    \draw[fedge, teal!70!black] (xe3) -- (gu22);

    \node[varnode, minimum size=12pt] (leg1) at (-0.8, -0.8) {};
    \node[font=\scriptsize, right=2pt] at (leg1.east) {Variable ($x_e$)};

    \node[constfactor, minimum size=10pt] (leg2) at (1.6, -0.8) {};
    \node[font=\scriptsize, right=2pt] at (leg2.east) {Constraint ($g_u$)};

    \node[unaryfactor, minimum size=9pt] (leg3) at (4.2, -0.8) {};
    \node[font=\scriptsize, right=2pt] at (leg3.east) {Unary ($f_e$)};

    \end{tikzpicture}%
    }%
    \caption{%
        \textbf{From hypergraph to factor graph.}
        \emph{Left:} A multi-view hypergraph with six detections across three views and three candidate hyperedges ($e_1$, $e_2$: cross-view correspondences; $e_3$: a partial hypothesis between two views).
        \emph{Right:} The corresponding factor graph for belief propagation.
        Each hyperedge becomes a binary variable node $x_e$, each detection induces a constraint factor $g_u$ enforcing exclusive assignment, and unary factors $f_e$ encode hyperedge plausibility.
        Variables are coupled when their hyperedges share a detection (e.g., $x_{e_1}$ and $x_{e_3}$ via $g_{u_1^1}$).
    }
    \label{fig:factor_graph}
\end{figure}

%% file: assets/figures/belief_propagation.tex
\begin{figure}[t]
    \centering
    \resizebox{\textwidth}{!}{%
    \begin{tikzpicture}[
        varnode/.style={circle, draw=black, fill=blue!15, minimum size=20pt, inner sep=0pt, font=\small},
        varnode_hl/.style={circle, draw=black, fill=blue!35, minimum size=20pt, inner sep=0pt, font=\small, line width=1.5pt},
        constfactor/.style={rectangle, draw=black, fill=red!25, minimum size=14pt, inner sep=2pt, font=\scriptsize},
        constfactor_hl/.style={rectangle, draw=black, fill=red!45, minimum size=14pt, inner sep=2pt, font=\scriptsize, line width=1.5pt},
        unaryfactor/.style={rectangle, draw=black, fill=green!20, minimum size=11pt, inner sep=1pt, font=\scriptsize},
        unaryfactor_hl/.style={rectangle, draw=black, fill=green!40, minimum size=11pt, inner sep=1pt, font=\scriptsize, line width=1.5pt},
        fedge/.style={thick, draw=black!20},
        msg_in/.style={-{Stealth[length=5pt]}, very thick, draw=purple!70!black},
        msg_out/.style={-{Stealth[length=7pt]}, line width=2.5pt, draw=orange!80!black},
        viewlabel/.style={font=\footnotesize\bfseries, text=gray},
    ]

    \node[viewlabel] at (3.0, 3.8) {Variable $\to$ Factor};

    \node[unaryfactor_hl] (fe1L) at (1.5, 3.2) {$f_{e_1}$};
    \node[unaryfactor] (fe2L) at (3.0, 3.2) {$f_{e_2}$};
    \node[unaryfactor] (fe3L) at (4.5, 3.2) {$f_{e_3}$};

    \node[varnode_hl] (xe1L) at (1.5, 2.0) {$x_{e_1}$};
    \node[varnode] (xe2L) at (3.0, 2.0) {$x_{e_2}$};
    \node[varnode] (xe3L) at (4.5, 2.0) {$x_{e_3}$};

    \draw[fedge] (fe2L) -- (xe2L);
    \draw[fedge] (fe3L) -- (xe3L);
    \draw[msg_in, dashed] (fe1L) -- (xe1L);

    \node[constfactor_hl] (gu11L) at (0.5, 0.4) {$g_{u_1^1}$};
    \node[constfactor] (gu21L) at (1.4, 0.4) {$g_{u_2^1}$};
    \node[constfactor] (gu12L) at (2.3, 0.4) {$g_{u_1^2}$};
    \node[constfactor] (gu22L) at (3.2, 0.4) {$g_{u_2^2}$};
    \node[constfactor] (gu13L) at (4.1, 0.4) {$g_{u_1^3}$};
    \node[constfactor] (gu23L) at (5.0, 0.4) {$g_{u_2^3}$};

    \draw[fedge] (xe1L) -- (gu12L);
    \draw[fedge] (xe1L) -- (gu13L);
    \draw[fedge] (xe2L) -- (gu21L);
    \draw[fedge] (xe2L) -- (gu22L);
    \draw[fedge] (xe2L) -- (gu23L);
    \draw[fedge] (xe3L) -- (gu11L);
    \draw[fedge] (xe3L) -- (gu22L);

    \draw[msg_in] (gu12L) -- (xe1L);
    \draw[msg_in] (gu13L) -- (xe1L);

    \draw[msg_out] (xe1L) -- (gu11L);

    \node[viewlabel] at (9.5, 3.8) {Factor $\to$ Variable};

    \node[unaryfactor] (fe1R) at (8.0, 3.2) {$f_{e_1}$};
    \node[unaryfactor] (fe2R) at (9.5, 3.2) {$f_{e_2}$};
    \node[unaryfactor] (fe3R) at (11.0, 3.2) {$f_{e_3}$};

    \node[varnode_hl] (xe1R) at (8.0, 2.0) {$x_{e_1}$};
    \node[varnode] (xe2R) at (9.5, 2.0) {$x_{e_2}$};
    \node[varnode] (xe3R) at (11.0, 2.0) {$x_{e_3}$};

    \draw[fedge] (fe1R) -- (xe1R);
    \draw[fedge] (fe2R) -- (xe2R);
    \draw[fedge] (fe3R) -- (xe3R);

    \node[constfactor_hl] (gu11R) at (7.0, 0.4) {$g_{u_1^1}$};
    \node[constfactor] (gu21R) at (7.9, 0.4) {$g_{u_2^1}$};
    \node[constfactor] (gu12R) at (8.8, 0.4) {$g_{u_1^2}$};
    \node[constfactor] (gu22R) at (9.7, 0.4) {$g_{u_2^2}$};
    \node[constfactor] (gu13R) at (10.6, 0.4) {$g_{u_1^3}$};
    \node[constfactor] (gu23R) at (11.5, 0.4) {$g_{u_2^3}$};

    \draw[fedge] (xe1R) -- (gu12R);
    \draw[fedge] (xe1R) -- (gu13R);
    \draw[fedge] (xe2R) -- (gu21R);
    \draw[fedge] (xe2R) -- (gu22R);
    \draw[fedge] (xe2R) -- (gu23R);
    \draw[fedge] (xe3R) -- (gu22R);

    \draw[msg_in] (xe3R) -- (gu11R);

    \draw[msg_out] (gu11R) -- (xe1R);

    \draw[msg_out] (1.5, -0.8) -- ++(1.0, 0);
    \node[font=\scriptsize, right=4pt] at (2.5, -0.8) {Computed message};

    \draw[msg_in] (5.5, -0.8) -- ++(1.0, 0);
    \node[font=\scriptsize, right=4pt] at (6.5, -0.8) {Incoming message};

    \draw[fedge] (9.5, -0.8) -- ++(1.0, 0);
    \node[font=\scriptsize, right=4pt] at (10.5, -0.8) {Inactive edge};

    \end{tikzpicture}%
    }%
    \caption{%
        \textbf{Belief propagation message passing on the factor graph.}
        \emph{Left:} A variable-to-factor message from $x_{e_1}$ to $g_{u_1^1}$.
        The message aggregates the unary log-potential from $f_{e_1}$ (dashed) and incoming messages from all other connected constraint factors.
        \emph{Right:} A factor-to-variable message from $g_{u_1^1}$ back to $x_{e_1}$.
        The constraint factor collects the incoming message from the competing hypothesis $x_{e_3}$ and sends a penalty reflecting how strongly $e_3$ claims the shared detection.
        Highlighted nodes are directly involved in the computation; dimmed edges are inactive.
    }
    \label{fig:bp_messages}
\end{figure}

%% file: supplementary/2_additional_methods.tex
\section{Additional Methodological and Implementation Details}
\label{app_sec:additional_methods}

\subsection{Normalized and Confidence-Weighted Hyperedge Compatibility}
\label{app_sec:hyperedge_score_normalization}

In the main manuscript (see \cref{subsec:hyperedge_scoring}), we define the hyperedge compatibility score as
\[
    s(e) = \exp(-\lambda \cdot C(e)),
\]
where $C(e)$ measures the geometric inconsistency of the detections grouped by hyperedge $e$.
In practice, we normalize this cost by the number of contributing joints and views, so that scores remain comparable across hyperedges of different cardinalities.
For the reprojection-based score, this corresponds to averaging the squared reprojection residuals over all scored joints and all detections contained in the hyperedge.

We additionally use the 2D detector confidence scores when computing the geometric cue. Let
$c^v_{i,j}\in[0,1]$ denote the confidence of joint $j$ in detection $U_i^v$. 
The reprojection cost is computed as a confidence-weighted average,
  \[
      \mathcal{C}(e)
      =
      \frac{
      \sum_{j}\sum_{U_i^v \in e} c^v_{i,j}\,
      \big\|\pi_v(\hat{\mathbf{y}}_j(e)) - \mathbf{u}^v_{i,j}\big\|^2
      }{
      \sum_{j}\sum_{U_i^v \in e} c^v_{i,j} + \epsilon
      },
  \]

where $\epsilon$ is a small constant for numerical stability. 
Thus, unreliable joints contribute less to the hyperedge score, while the normalization prevents larger hyperedges from being penalized solely because they contain more observations. 
The resulting compatibility score $s(e)$ is used both for pruning geometrically implausible hyperedges and as the unary score in the ILP and BP objectives.

\subsection{Confidence-Weighted Triangulation}
\label{app_sec:weighted_triangulation}

After correspondence estimation, each selected non-singleton hyperedge $e\in\cE^\star$ is converted into a 3D pose by triangulating its corresponding 2D joints.
In the main manuscript, we denote this triangulated joint by $\hat{\mathbf{y}}_j(e)$.
In practice, we use the confidence scores produced by the off-the-shelf 2D pose detector to perform confidence-weighted triangulation \cite{iskakov_LearnableTriangulationHuman_2019,dong2019fast}.

Let $c_{i,j}^v\in[0,1]$ denote the confidence of joint $j$ in detection $U_i^v$.
For a selected hyperedge $e$ and joint $j$, we triangulate $\hat{\mathbf{y}}_j(e)\in\mathbb{R}^3$ from the observations $\{\mathbf{u}_{i,j}^v : U_i^v\in e\}$ by solving the weighted least-squares problem
\[
    \hat{\mathbf{y}}_j(e)
    =
    \argmin_{\mathbf{y}\in\mathbb{R}^3}
    \sum_{U_i^v\in e}
    w_{i,j}^v
    \left\|
        \pi_v(\mathbf{y}) - \mathbf{u}_{i,j}^v
    \right\|^2,
\]
where $w_{i,j}^v$ is derived from the detector confidence $c_{i,j}^v$.
In our implementation, we set $w_{i,j}^v = c_{i,j}^v$ and solve the corresponding weighted DLT system.

To reduce the influence of occasional outlier views, we apply a simple reprojection-error-based robustification step \cite{hartley2003multiple,bartol2022generalizable}.
After an initial confidence-weighted triangulation, we compute the per-view reprojection residual
\[
    r_{i,j}^v(e)
    =
    \left\|
        \pi_v(\hat{\mathbf{y}}_j(e)) - \mathbf{u}_{i,j}^v
    \right\|^2 .
\]
If at least three views are available for joint $j$, we discard the observation with the largest residual, and recompute $\hat{\mathbf{y}}_j(e)$ from the remaining observations.

\subsection{BP Update Equations}
\label{app_subsec:bp_updates}

With $\eta=0$ and unary log-potential
$\phi_e \coloneqq \beta\,(s(e)-\gamma)$, the variable-to-factor and
factor-to-variable log-ratio updates are
\begin{align}
  m_{e \to u} &= \phi_e + \sum_{u' \in e \setminus \{u\}} n_{u' \to e},
  \label{eq:app_m_update}
  \\
  n_{u \to e} &= -\log\Biggl(
    1 + \sum_{e' \in \cE(u) \setminus \{e\}}
    \exp\bigl(m_{e' \to u}\bigr)
  \Biggr).
  \label{eq:app_n_update}
\end{align}

\subsection{Probabilistic Relaxation via Belief Propagation}
\label{app_sec:bp}

We solve the relaxed factor-graph formulation using damped loopy belief propagation.
\Cref{alg:simple_bp} summarizes the procedure.
We iteratively pass messages to estimate soft hyperedge beliefs $b_e$, and then greedily decode the final beliefs into a valid non-overlapping discrete selection.
In practice, we terminate once the beliefs stabilize, i.e., when
$\max_{e\in\mathcal E}|b_e^{t+1}-b_e^t|<\varepsilon$.
Empirically, on the CMU Panoptic \cite{panoptic} dataset and with $\varepsilon=10^{-3}$, the algorithm converges within three iterations for all evaluated instances.
In \cref{alg:simple_bp}, we set $\gamma = 2.5$, $\beta = 0.5$, and $\alpha = 0.25$ for all experiments.

\begin{algorithm}[h!t]
\caption{Belief Propagation for Hypergraph Matching}
\label{alg:simple_bp}
\KwIn{Hypergraph incidences $(u,e)$, hyperedge scores $s(e)$, sparsity penalty $\gamma$, inverse temperature $\beta$, damping $\alpha$, maximum iterations $T$, tolerance $\varepsilon$.}
\KwOut{Beliefs $b_e$ and discrete selection $\bx$.}

Initialize variable-to-factor messages $m_{e\to u}^{0} \gets 0$\;
Initialize factor-to-variable messages $n_{u\to e}^{0} \gets 0$\;
Set log-potentials $\phi_e \gets \beta\,(s(e) - \gamma)$\;

\For{$t=0,\dots,T-1$}{
    \tcp{Variable-to-factor update}
    \ForEach{incidence $(u,e)$}{
        \[
        \tilde{m}_{e\to u}^{t+1}
        \gets
        \phi_e + \sum_{u' \in e \setminus \{u\}} n_{u'\to e}^{t}
        \]
        \[
        m_{e\to u}^{t+1}
        \gets
        \alpha \tilde{m}_{e\to u}^{t+1} + (1-\alpha)m_{e\to u}^{t}
        \]
    }

    \tcp{Factor-to-variable update}
    \ForEach{incidence $(u,e)$}{
        \eIf{sum-product}{
            \[
            \tilde{n}_{u\to e}^{t+1}
            \gets
            -\log\!\left(1 + \sum_{e' \in \cE(u),\; e'\neq e} \exp(m_{e'\to u}^{t+1})\right)
            \]
        }{
            \[
            \tilde{n}_{u\to e}^{t+1}
            \gets
            -\max\!\left(0,\max_{e' \in \cE(u),\; e'\neq e}m_{e'\to u}^{t+1}\right)
            \]
        }
        \[
        n_{u\to e}^{t+1}
        \gets
        \alpha \tilde{n}_{u\to e}^{t+1} + (1-\alpha)n_{u\to e}^{t}
        \]
    }

    \tcp{Belief computation}
    \ForEach{hyperedge $e$}{
        \[
        b_e^{t+1} \gets \sigma\!\left(\phi_e + \sum_{u\in e}n_{u\to e}^{t+1}\right)
        \]
    }

    \If{$\max_e |b_e^{t+1} - b_e^t| < \varepsilon$}{
        \textbf{break}
    }
}

\tcp{Greedy decoding}
$\bx \gets \textsc{GreedyRounding}(b)$\;

\Return{$b, \bx$}\;
\end{algorithm}

\subsection{2D Pose Estimator}
\label{app_sec:2d_pose_estimator}

For our top-down 2D pose estimation pipeline, we use ViTPose++ \cite{xu2022vitpose} with the \texttt{ViTPose-plus-huge} variant for pose estimation, together with RT-DETRv2 \cite{lv2024rt} (\texttt{r101} variant) for person detection.
Both models are initialized from their official pretrained weights.

In \cref{app_sec:ablation_2d_pose_prior}, we present experimental results analyzing the influence of various variants of the 2D person detector and 2D human pose estimator on overall performance.
We also provide results using the default detectors used to create the 2D human pose labels in SelfPose3d \cite{srivastav_SelfPose3dSelfSupervisedMultiPerson_2024}.

\subsection{ILP Solver}
\label{app_sec:ilp_solver}

To solve the optimization problem using Integer Linear Programming (ILP), we implement our framework using the open-source Python library \texttt{PuLP} \cite{mitchell2011pulp}. 
The underlying optimization process is driven by PuLP's \texttt{Coin-or branch and cut} (CBC) solver \cite{forrest2005cbc}.

\subsection{Hyperparameter Details}
\label{app_sec:hyperparameter_details}

\Cref{app_tab:hyperparameters} summarizes the method hyperparameters $(\gamma,\lambda,\tau,\ldots)$ together with the BP-specific solver settings used in our experiments across the CMU Panoptic \cite{panoptic}, Shelf \cite{belagiannis_3DPictorialStructures_2014}, and Campus \cite{belagiannis_3DPictorialStructures_2014} datasets.
We apply a consistent set of parameters across datasets and camera arrangements.
The only exception is the Campus dataset \cite{belagiannis_3DPictorialStructures_2014}, where a larger $\tau$ threshold is implemented to compensate for less accurate 2D pose detections, a consequence of the dataset's outdoor environment and lower image resolution.

\begin{table}[h!t]
\centering
\caption{Summary of hyperparameters utilized for the CMU Panoptic \cite{panoptic}, Shelf \cite{belagiannis_3DPictorialStructures_2014}, and Campus \cite{belagiannis_3DPictorialStructures_2014} datasets.}
\label{app_tab:hyperparameters}
\begin{tabular*}{\textwidth}{@{\extracolsep{\fill}}lccc@{}}
\toprule
Hyperparameter   & Panoptic  & Shelf     & Campus    \\ 
\midrule
$\gamma$           & 2.5       & 2.5       & 2.5       \\
$\lambda$          & 0.01      & 0.01      & 0.01      \\
$\tau$             & $(32\,\mathrm{px})^2$  & $(32\,\mathrm{px})^2$  & $(64\,\mathrm{px})^2$  \\
$\beta$ (BP)       & 0.5       & 0.5       & 0.5       \\
$\alpha$ (BP)      & 0.25      & 0.25      & 0.25      \\
Max. BP iterations & 10        & 10        & 10        \\
$\varepsilon$ (BP) & $10^{-3}$ & $10^{-3}$ & $10^{-3}$ \\
\bottomrule
\end{tabular*}
\end{table}

\subsection{Dataset Details}
\label{app_sec:evaluation_details}

\paragraph{CMU Panoptic}
Aligning with established evaluation protocols, we partition the CMU Panoptic dataset \cite{panoptic} identically to previous methodologies \cite{tu_VoxelPoseMulticamera3D_2020, ye_FasterVoxelPoseRealtime_2022, wu_GraphBased3DMultiPerson_2021, srivastav_SelfPose3dSelfSupervisedMultiPerson_2024, liao_MultipleViewGeometry_2024}. 
Hyperparameter tuning is conducted using a specific subset of sequences: \texttt{160422\_ultimatum1}, \texttt{160224\_haggling1}, \texttt{160226\_haggling1}, \texttt{161202\_haggling1}, \texttt{160906\_ian1}, \texttt{160906\_ian2}, \texttt{160906\_ian3}, \texttt{160906\_band1}, and \texttt{160906\_band2}. 
The final assessment is performed on the sequences: \texttt{160422\_haggling1}, \texttt{160906\_pizza1}, \texttt{160906\_ian5}, and \texttt{160906\_band4}.
Consistent with prior works, we extract every 12th frame from these test sequences, yielding 2,580 frames for the final evaluation.

\paragraph{Shelf Dataset}
Evaluation on the Shelf dataset \cite{belagiannis_3DPictorialStructures_2014} is conducted using frames from frame index \texttt{300} to \texttt{600}, which mirrors standard conventions established by recent literature \cite{tu_VoxelPoseMulticamera3D_2020, ye_FasterVoxelPoseRealtime_2022, wu_GraphBased3DMultiPerson_2021, srivastav_SelfPose3dSelfSupervisedMultiPerson_2024}.

\paragraph{Campus Dataset}
For the Campus dataset \cite{belagiannis_3DPictorialStructures_2014}, we also adhere to the standard benchmark splits \cite{tu_VoxelPoseMulticamera3D_2020, ye_FasterVoxelPoseRealtime_2022, wu_GraphBased3DMultiPerson_2021, srivastav_SelfPose3dSelfSupervisedMultiPerson_2024}. 
The model is evaluated on frames \texttt{350} through \texttt{470}, as well as \texttt{650} through \texttt{750}. 

\subsection{Camera Configurations on CMU Panoptic}
\label{app_subsec:panoptic_generalization}

We assess the impact of varying camera configurations in our generalization experiments, utilizing the CMU Panoptic dataset \cite{panoptic}. 
We evaluate both the baseline models and our proposed approach under different camera arrangements and varying camera counts.

For these tests, we adopt the experimental framework introduced by Liao et al. (MVGFormer) \cite{liao_MultipleViewGeometry_2024}, a standard subsequently also used by Chharia et al. (MV-SSM) \cite{chharia_MVSSMMultiViewState_2025}. 
The specific camera permutations deployed in our study are detailed in \cref{tab:camera_setups}. 
For the \texttt{CMU4} configuration, we select the first four cameras, rather than the complete set of ten originally used in MVGFormer \cite{liao_MultipleViewGeometry_2024} to provide more settings with fewer cameras.

\begin{table}[ht]
    \caption{Specification of camera setups used in our generalization studies. 
    We list the specific Camera IDs alongside the total number of views for each configuration on the CMU Panoptic dataset \cite{panoptic}.}
    \label{tab:camera_setups}
    \centering
    \begin{tabularx}{\textwidth}{@{} l X c @{}}
        \toprule
        \textbf{Setup Name} & \textbf{Camera IDs} & \textbf{\# Views} \\
        \midrule
        CMU0 & 3, 6, 12, 13, 23 & 5 \\
        CMU0 w/ 2 extra & 3, 6, 12, 13, 23, 10, 16 & 7 \\
        CMU0($K$) & First $K$ cameras from ``CMU0 w/ 2 extra'' & $K$ \\
        \midrule
        CMU1 & 1, 2, 3, 4, 6, 7, 10 & 7 \\
        CMU2 & 12, 16, 18, 19, 22, 23, 30 & 7 \\
        CMU3 & 10, 12, 16, 18 & 4 \\
        CMU4 & 6, 7, 10, 12 & 4 \\
        \bottomrule
    \end{tabularx}
\end{table}

\subsection{Evaluation Metrics}
\label{app_sec:evaluation_metrics}

We evaluate the quantitative performance of our 3D pose estimation framework using three primary metrics: Mean Per-Joint Position Error (MPJPE), Average Precision (AP), and Percentage of Correct Parts (PCP).
These are the standard metrics used in each dataset.

\begin{itemize}[leftmargin=*,itemsep=0.2em]
    \item \textbf{MPJPE (Mean Per Joint Position Error):}
          MPJPE computes the Euclidean distance (reported in millimeters) between the predicted joint coordinates and their corresponding ground-truth locations. 
          For any individual pose, this metric calculates the mean error across all visible GT joints:
          \begin{equation}
              E_{pose} = \frac{1}{J} \sum_{j=1}^{J} | \mathbf{p}_j - \mathbf{g}_j |_2 ,
          \end{equation}
          where $\mathbf{p}_j$ and $\mathbf{g}_j$ denote the estimated and actual 3D coordinates of the $j$-th joint, respectively, and $J$ represents the total count of visible GT joints.

    \item \textbf{Average Precision (AP) and Recall:}
          A predicted human pose qualifies as a True Positive if its overarching pose-level error ($E_{pose}$) falls strictly below a defined threshold $\kappa$. 
          We compute both Average Precision (AP) and Recall across a spectrum of thresholds, specifically $\kappa \in \{25, 50, \dots, 150\}$ mm. 

          Furthermore, we provide \textbf{Recall$_{500\text{mm}}$}, which highlights the fraction of ground-truth subjects successfully localized within a broader 500\,mm error radius.

    \item \textbf{PCP (Percentage of Correct Parts):}
          This metric calculates the ratio of correctly predicted limbs. 
          A specific limb (defined by a starting joint $s$ and an ending joint $e$) is classified as ``correct'' provided that the mean positional error of its two endpoints does not exceed 50\% of the actual ground-truth limb length.
          A limb satisfies this condition if:
          \begin{equation}
              \frac{\| \mathbf{p}_s - \mathbf{g}_s \|_2 + \| \mathbf{p}_e - \mathbf{g}_e \|_2}{2} \leq 0.5 \cdot \| \mathbf{g}_s - \mathbf{g}_e \|_2 .
          \end{equation}

\end{itemize}

We note that PCP does not evaluate the precision of the predicted poses and only considers the recall of correctly localized limbs.
This is because both the Shelf and Campus dataset \cite{belagiannis_3DPictorialStructures_2014} are not annotated exhaustively, and thus the precision of the predicted poses cannot be reliably measured.

%% file: supplementary/3_additional_experiments.tex
\section{Additional Experiments}
\label{app_sec:additional_experiments}

\subsection{Extended CMU Panoptic Comparison}
\label{app_sec:panoptic_full}

For completeness, \cref{app_tab:panoptic_full} reports the CMU Panoptic comparison from \cref{tab:panoptic_main} of the main manuscript with the additional fully supervised baselines (VoxelPose~\cite{tu_VoxelPoseMulticamera3D_2020}, MvP~\cite{wang_DirectMultiviewMultiperson_2021}, Faster VoxelPose~\cite{ye_FasterVoxelPoseRealtime_2022}, MVGFormer~\cite{liao_MultipleViewGeometry_2024}, and MV-SSM~\cite{chharia_MVSSMMultiViewState_2025}) that were omitted from the main paper for space.
Including these methods does not alter the per-category best results as the highlighted entries in \cref{tab:panoptic_main} remain the best within their respective supervision groups.

\input{assets/tables/panoptic_results_full.tex}

\subsection{Extended Shelf and Campus Comparison}
\label{app_sec:shelf_campus_full}

For completeness, \cref{app_tab:results_shelf_campus_full} reports the Shelf and Campus comparison from \cref{tab:results_shelf_campus} of the main manuscript with the additional fully supervised baselines (MvP~\cite{wang_DirectMultiviewMultiperson_2021} and Faster VoxelPose~\cite{ye_FasterVoxelPoseRealtime_2022}) that were omitted from the main paper for space.
Including these methods reattributes a small number of per-category best results among the fully supervised group: Faster VoxelPose attains the strongest Shelf-A1 PCP (99.4 vs.\ 99.3 for the main-paper baselines), and MvP the strongest Campus-A1 PCP (98.2 vs.\ 97.7 for TEMPO).
The relative ordering across supervision categories, and all best results in the self-supervised and optimization-based groups, are unchanged.

\input{assets/tables/shelf_campus_results_full.tex}

\subsection{Pruning robustness}
\label{app_subsec:pruning_robustness}

The geometric pruning discards hyperedges whose mean squared reprojection error exceeds $\tau$, controlling the recall/precision balance of the candidate set fed to the matching solver.
\Cref{tab:ablation_pruning} sweeps $\sqrt{\tau}$ from 16 to 128 pixels on a subsample of the Panoptic training sequences (every 128\textsuperscript{th} frame; 8 sequences, 545 frames), holding all other parameters at their default values. 
The mAP curve is unimodal at our default $\tau=(32\,\mathrm{px})^2$: tightening to $\tau=(16\,\mathrm{px})^2$ shifts the operating point toward precision -- winning AP$_{25}$ by 1.5 points but losing 2.3 AP$_{50}$ -- while loosening to $\tau \ge (64\,\mathrm{px})^2$ admits noisy candidates that progressively erode AP$_{25}$ from 53.92 down to 41.14 at $\tau=(128\,\mathrm{px})^2$.
Our default sits at the mAP peak and is the joint optimum for AP$_{50}$, AP$_{100}$, and Recall$_{500\text{mm}}$.

\begin{table}[h!t]
\caption{
\textbf{Pruning-threshold ablation on the CMU Panoptic train split \cite{panoptic}.}
Mean-squared reprojection threshold $\tau$ governs which hyperedges survive geometric pruning. All other parameters are kept at their default values (see \cref{app_sec:hyperparameter_details}).
}
\label{tab:ablation_pruning}
\centering
\small
\setlength{\tabcolsep}{4pt}
\begin{tabular*}{\textwidth}{@{\extracolsep{\fill}} @{\hspace{2pt}}l
S[table-format=2.2] S[table-format=2.2]
S[table-format=2.2] S[table-format=2.2]
S[table-format=2.2] S[table-format=2.2] }
    \toprule
    \multirow{2}{*}{\textbf{Threshold} $\tau$}
        & \multicolumn{4}{c}{\textbf{Average Precision (AP)} ($\uparrow$)}
        & {\textbf{mAP}}
        & {\textbf{Recall}} \\
    \cmidrule(lr){2-5} \cmidrule(lr){6-6} \cmidrule(lr){7-7}
        & {25} & {50} & {100} & {150} & {25--150} & {@500} \\
    \midrule
    $(16\,\mathrm{px})^2$               & 55.43 & 87.51 & 95.37 & 97.28 & 87.43 & 99.38 \\
    $(32\,\mathrm{px})^2$ (default)     & 53.92 & 89.78 & 96.51 & 98.15 & 88.35 & 99.55 \\
    $(64\,\mathrm{px})^2$               & 45.13 & 89.34 & 96.47 & 98.38 & 86.83 & 99.32 \\
    $(96\,\mathrm{px})^2$               & 42.67 & 88.53 & 96.24 & 98.33 & 86.11 & 99.27 \\
    $(128\,\mathrm{px})^2$              & 41.14 & 88.26 & 96.26 & 98.28 & 85.76 & 99.27 \\
    \bottomrule
\end{tabular*}
\end{table}

\subsection{The Matching Solver Matters}

To verify that joint matching provides additional benefit beyond what the geometric pruning already provides, we replace the optimizer with a naive greedy baseline. 
Greedy selects the highest-scoring hyperedge whose nodes are still unclaimed; ILP solves for the joint optimum, while BP provides a relaxed message-passing approximation. 
Greedy, given the same pruned graph, underperforms the optimizers by 5.2 AP$_{50}$ and 3.3\,mm MPJPE.

\begin{table}[h!t]
  \caption{
  \textbf{Solver evaluation on the CMU Panoptic test split.}
  The greedy heuristic, applied to the same pruned graph, cannot recover the ILP joint optimum.
  }
  \label{tab:ablation_solver}
  \centering
  \small
  \setlength{\tabcolsep}{4pt}
  \begin{tabular*}{\textwidth}{@{\extracolsep{\fill}} @{\hspace{2pt}}l
  S[table-format=2.2] S[table-format=2.2]
  S[table-format=2.2] S[table-format=2.2]
  S[table-format=2.2] S[table-format=2.2] }
      \toprule
      \multirow{2}{*}{\textbf{Solver}}
          & \multicolumn{4}{c}{\textbf{Average Precision (AP)} ($\uparrow$)}
          & {\textbf{Recall} ($\uparrow$)}
          & {\textbf{Error} ($\downarrow$)} \\
      \cmidrule(lr){2-5} \cmidrule(lr){6-6} \cmidrule(lr){7-7}
          & {25} & {50} & {100} & {150} & {@500} & {MPJPE} \\
      \midrule
      Greedy                       & 61.30 & 93.19 & 97.51 & 98.02 & 99.80 & 26.05 \\
      \name-ILP                    & 66.70 & 98.23 & 99.43 & 99.62 & 99.81 & 22.78 \\
      \name-BP                     & 68.88 & 98.37 & 99.42 & 99.61 & 99.81 & 22.78 \\
      \bottomrule
  \end{tabular*}
\end{table}

\subsection{Robustness of The Matching Penalty}

The ILP penalty $\gamma$ is the only optimizer-side hyperparameter. Sweeping $\gamma \in [2.0, 3.0]$ at the default $\tau=(32\,\mathrm{px})^2$ produces similar AP at every reported threshold, indicating that the joint optimum is well-defined and insensitive to the penalty within a wide band. We use $\gamma = 2.5$ throughout. $\lambda$, in contrast, is cost-function-specific (\cref{tab:ablation_score}) and is tuned per cost function rather than treated as a free parameter.

\subsection{Generalization: Effect of Varying Camera Numbers}
\label{app_sec:additional_qualitative_results}

We study the effect of varying the number of available cameras on the standard CMU Panoptic setup (CMU0) by progressively adding views.
\Cref{app_tab:generalization_cam_count} reports mAP and Recall$_{500\text{mm}}$ for subsets of 3, 4, 6, and 7 cameras.
As expected, all methods benefit from additional views, yet \textsc{\name} exhibits the most consistent gains.
Even in the challenging 3-camera regime, \name-ILP achieves 72.01 mAP, outperforming MvPose (58.05) by a large margin and surpassing the self-supervised SelfPose3d (66.42), which requires training supervision.
As the number of views increases, our hypergraph formulation leverages a richer multi-view consistency signal: at 7 cameras, \name-BP achieves 95.17 mAP, compared to 89.80 for MvPose.
Notably, the gap between \name-ILP and \name-BP remains small across all configurations, suggesting that the BP relaxation provides a close approximation to the exact ILP solution under different evaluated settings.

\begin{table}[h!t]
    \caption{\textbf{Generalization across camera numbers on CMU Panoptic \cite{panoptic}.}
    We report mean Average Precision (mAP) and Recall$_{500\text{mm}}$ (Rec.).
    The number of available cameras for each setup is indicated in parentheses.
    Best results are \textbf{highlighted}}
    \label{app_tab:generalization_cam_count}
    \centering
    \small
    \setlength{\tabcolsep}{3pt}
    \begin{tabular*}{\textwidth}{ @{\extracolsep{\fill}} l l *{8}{S[table-format=2.2]} }
        \toprule
        \multirow{2}{*}{\textbf{Type}} & \multirow{2}{*}{\textbf{Method}} 
        & \multicolumn{2}{c}{\textbf{CMU0(3)}} 
        & \multicolumn{2}{c}{\textbf{CMU0(4)}} 
        & \multicolumn{2}{c}{\textbf{CMU0(6)}} 
        & \multicolumn{2}{c}{\textbf{CMU0(7)}} \\
        \cmidrule(lr){3-4} \cmidrule(lr){5-6} \cmidrule(lr){7-8} \cmidrule(lr){9-10}
        & & {mAP} & {Rec.} & {mAP} & {Rec.} & {mAP} & {Rec.} & {mAP} & {Rec.} \\
        \midrule
        Self-Sup. & SelfPose3d \cite{srivastav_SelfPose3dSelfSupervisedMultiPerson_2024}
        & 66.42 & 93.40 & 86.59 & {\bfseries 99.44} & 78.77 & 99.41 & 78.73 & 99.20 \\
        \midrule
        \multirow{3}{*}{Optim.}
        & MvPose \cite{dong2019fast}
        & 58.05 & 96.65 & 79.13 & 98.99 & 89.00 & 99.68 & 89.80 & 99.73 \\
        & \name-ILP (Ours)
        & {\bfseries 72.01} & {\bfseries 97.92} & {\bfseries 86.65} & {\bfseries 99.44} & 94.34 & {\bfseries 99.86} & 95.06 & {\bfseries 99.86} \\
        & \name-BP (Ours)
        & 70.01 & {\bfseries 97.92} & 86.52 & 99.43 & {\bfseries 94.58} & {\bfseries 99.86} & {\bfseries 95.17} & {\bfseries 99.86} \\
        \bottomrule
    \end{tabular*}
\end{table}

\subsection{Effect of Different Cost Functions}
\label{app_sec:ablation_cost_functions}

We evaluate four geometric cost functions $\mathcal{C}$ for hyperedge weighting:

\begin{enumerate}[leftmargin=*,itemsep=0.2em]
    \item The \emph{epipolar distance} \cite{hartley2003multiple} computes the average absolute epipolar constraint over all view pairs in normalized camera coordinates.
    \item The \emph{Sampson error} \cite{hartley2003multiple} normalizes the squared epipolar constraint by the epipolar line gradients, approximating the geometric distance to the epipolar line in pixel space using the fundamental matrix.
    \item The \emph{trifocal tensor} \cite{hartley1997lines} generalizes to view triples by evaluating an algebraic constraint derived from the trifocal tensor using the known projection matrices, averaged over all triples within the hyperedge; for hyperedges containing only two views, we fall back to the Sampson error.
    \item The \emph{reprojection error} (our default) \cite{hartley2003multiple} triangulates a 3D point and measures the average squared pixel distance to the observed detections.
\end{enumerate}

\Cref{tab:ablation_score} compares geometric and photometric cost functions for hyperedge weighting. 
Each cost function uses its own tuned $\lambda$. 
Within the geometric scores, the epipolar distance performs worse than the alternatives. 
Sampson error and the trifocal tensor yield similar results, whereas the reprojection error provides the strongest geometric baseline, especially at AP$_{25}$ and in MPJPE. 
For the photometric cues, ``alone'' applies a feature-distance threshold to prune candidate hyperedges and uses the same cosine-similarity score for matching; ``$\times$'' applies the photometric score multiplicatively on top of the geometric pruning. 
Person re-ID features \cite{zhou2021learning} alone produce a viable but inferior matching cue ($-$4.2 AP$_{50}$, $+$1.6\,mm MPJPE vs.\ reprojection), confirming that appearance carries some multi-view correspondence signal but is dominated by geometric consistency.
DinoV2 features \cite{oquab2024dinov2} alone fail ($-$65 AP$_{50}$, $+$34\,mm MPJPE), suggesting that generic self-supervised features might not be discriminative enough between people for this task.
When applied on top of the geometric pruning, photometric scores have no measurable effect.

 \begin{table}[h!t]
      \caption{
      \textbf{Different cost functions on CMU Panoptic \cite{panoptic}.}
      We compare different cost functions $\mathcal{C}$ used for hyperedge weighting in \name-ILP.
      The first block lists geometric cost functions.
      The second extends the ablation with photometric alternatives: person re-identification features (OSNet-AIN \cite{zhou2021learning}, 512-d) and DinoV2 features (ViT-B/14 \cite{oquab2024dinov2}, 768-d).
      ``Alone'' uses the photometric cue both for pruning and for matching score; ``$\times$'' multiplies the photometric score onto the geometric one on the geometrically-pruned graph.
      }
      \label{tab:ablation_score}
      \centering
      \small
      \setlength{\tabcolsep}{3pt}
      \begin{tabular*}{\textwidth}{@{\extracolsep{\fill}} @{\hspace{2pt}}l
      S[table-format=2.2] S[table-format=2.2] S[table-format=2.2]
      S[table-format=3.2] S[table-format=2.2] S[table-format=3.2] }
          \toprule
          \multirow{2}{*}{\textbf{Score Function}}
              & \multicolumn{4}{c}{\textbf{Average Precision (AP)} ($\uparrow$)}
              & {\textbf{Recall} ($\uparrow$)}
              & {\textbf{Error} ($\downarrow$)} \\
          \cmidrule(lr){2-5} \cmidrule(lr){6-6} \cmidrule(lr){7-7}
              & {25} & {50} & {100} & {150} & {@500} & {MPJPE} \\
          \midrule
          Epipolar Distance               & 62.93 & 96.22 & 97.45 & 97.70 & 97.98 & 23.01 \\
          Sampson Error                   & 63.46 & 97.77 & 99.03 & 99.21 & 99.42 & 23.02 \\
          Trifocal Tensor                 & 63.54 & 97.93 & 99.26 & 99.49 & 99.69 & 23.08 \\
          Reprojection Error (default)    & 66.70 & 98.23 & 99.43 & 99.62 & 99.81 & 22.78 \\
          \midrule
          Person Re-ID (alone)            & 57.92 & 94.00 & 95.64 & 95.80 & 96.45 & 24.35 \\
          DinoV2 (alone)                  & 22.35 & 32.86 & 34.30 & 34.80 & 44.94 & 57.26 \\
          Reprojection $\times$ Re-ID     & 66.69 & 98.25 & 99.45 & 99.64 & 99.82 & 22.78 \\
          Reprojection $\times$ DinoV2    & 66.68 & 98.25 & 99.45 & 99.63 & 99.82 & 22.79 \\
          \bottomrule
      \end{tabular*}
  \end{table}

\subsection{Effect of Different 2D Pose Estimation Prior}
\label{app_sec:ablation_2d_pose_prior}

We investigate the sensitivity of \name-ILP to the quality of the 2D pose prior by varying the backbone capacity of both the person detector and the top-down pose estimator. 
Specifically, we evaluate four configurations on the CMU Panoptic dataset: the \texttt{K-RCNN}+\texttt{HRNet} setup used by SelfPose3d \cite{srivastav_SelfPose3dSelfSupervisedMultiPerson_2024} to generate its pseudo labels, a lightweight setup using a ResNet-18 backbone for RT-DETRv2~\cite{lv2024rt} (\texttt{R-18}) paired with the small version of the ViTPose++ pose estimator~\cite{xu2022vitpose} (\texttt{Small}), an intermediate setup with a ResNet-50 backbone (\texttt{R-50}) and the base pose-estimator variant (\texttt{Base}), and our default configuration using ResNet-101 (\texttt{R-101}) with the huge pose-estimator variant (\texttt{Huge}).

\Cref{tab:ablation_2d_prior} reports the results.
All configurations achieve comparable performance at relaxed AP thresholds (AP$_{100}$ and above) and Recall$_{500\text{mm}}$, indicating that our hypergraph formulation is robust to moderate variations in 2D input quality.
The primary differences emerge at the strictest threshold (AP$_{25}$), where the stronger backbone yields a $+$3.25 improvement (66.70 vs.\ 63.45), and in MPJPE, where more precise 2D localizations translate into lower triangulation error (22.78\,mm vs.\ 23.43\,mm).
This indicates that while the correspondence matching stage of \textsc{\name} is largely insensitive to the 2D backbone, the final triangulation accuracy naturally benefits from more precise 2D keypoint detections.

\begin{table}[htbp]
    \centering
    \caption{
    \textbf{Effect of different 2D pose estimation priors on the CMU Panoptic dataset \cite{panoptic}.}
    We vary the person detector (\texttt{K-RCNN}~\cite{he2017mask}, \texttt{RT-DETR} variants~\cite{lv2024rt}) and pose estimator (\texttt{HRNet}~\cite{sun2019deep}, \texttt{ViTPose} variants~\cite{xu2022vitpose}) backbones in \name-ILP.
    The \texttt{K-RCNN}+\texttt{HRNet} configuration corresponds to the 2D prior used by SelfPose3d \cite{srivastav_SelfPose3dSelfSupervisedMultiPerson_2024} to generate pseudo labels.
    Best results are highlighted in \textbf{bold}.
    }
    \label{tab:ablation_2d_prior}
    \setlength{\tabcolsep}{3pt}
    \begin{tabular*}{\textwidth}{@{\extracolsep{\fill}} @{\hspace{2pt}}ll
      S[table-format=2.2, detect-weight, mode=text]
      S[table-format=2.2, detect-weight, mode=text]
      S[table-format=2.2, detect-weight, mode=text]
      S[table-format=3.2, detect-weight, mode=text]
      S[table-format=2.2, detect-weight, mode=text]
      S[table-format=3.2, detect-weight, mode=text] }
        \toprule
        \multirow{2}{*}{\textbf{Person Detector}}
            & \multirow{2}{*}{\textbf{Pose Estimator}}
            & \multicolumn{4}{c}{\textbf{Average Precision (AP)} ($\uparrow$)}
            & {\textbf{Recall} ($\uparrow$)}
            & {\textbf{Error} ($\downarrow$)} \\
        \cmidrule(lr){3-6} \cmidrule(lr){7-7} \cmidrule(lr){8-8}
            & & {25} & {50} & {100} & {150} & {@500} & {MPJPE} \\
        \midrule
        \texttt{K-RCNN} & \texttt{HRNet}  & 60.32 & 97.62 & 99.23 & 99.47 & 99.73 & 23.91 \\
        \texttt{R-18}   & \texttt{Small}  & 63.45 & 97.79 & 99.26 & 99.51 & 99.75 & 23.43 \\
        \texttt{R-50}   & \texttt{Base}   & 62.97 & 98.22 & 99.33 & 99.52 & \bfseries 99.81 & 23.67 \\
        \texttt{R-101}  & \texttt{Huge}   & \bfseries 66.70 & \bfseries 98.23 & \bfseries 99.43 & \bfseries 99.62 & \bfseries 99.81 & \bfseries 22.78 \\
        \bottomrule
    \end{tabular*}
\end{table}

\subsection{Cross-View Consensus Accuracy}
\label{app_subsec:cross_view_ass_acc}

We isolate the quality of the matching stage by checking, for every ground-truth person, whether \textbf{all} of their 2D appearances are assigned into the same predicted hyperedge (correspondence set).
A person is considered correctly associated when no view is split off into a different set and no other person's detections are merged in. 
Both methods are evaluated on the same input 2D detections.

\name correctly associates 94.9\% of persons perfectly compared to 87.5\% for MvPose; a direct benefit of jointly reasoning over all views as hyperedges rather than fusing pairwise matches.

\begin{table}[ht]
  \centering
  \small
  \caption{\textbf{Cross-view consensus accuracy on CMU Panoptic test \cite{panoptic}}
  Fraction of ground-truth persons where all views are correctly assigned to a single hyperedge (correspondence set).}
  \label{tab:association_accuracy}
  \begin{tabular}{lc}
  \toprule
  \textbf{Method} & \textbf{Correctly Associated} \\
  \midrule
  MvPose~\cite{dong2019fast} & 87.5\% \\
  \name-ILP (Ours)             & \textbf{94.9\%} \\
  \bottomrule
  \end{tabular}
\end{table}

\subsection{Hypergraph Construction}
\label{app_subsec:hierarchical_hypergraph_construction}

Materializing the full $V$-partite candidate set scales as $O(N^V)$ and exhausts a 16 GB GPU around $V=8$ (\cref{tab:memory_profile}); we instead build the hypergraph lazily. 
Order-2 hyperedges are formed from every pair of detections across distinct views and pruned by the threshold $\tau$; each subsequent order is constructed by extending the surviving $k$-hyperedges with one detection from a new view and re-pruning under the same threshold. 
Peak memory, therefore, scales with the number of geometrically plausible survivors at each level rather than with the combinatorial space, enabling inference at $V \ge 10$ on commodity hardware.

\begin{table}[h!t]
  \caption{
  \textbf{Peak GPU memory of hyperedge construction across camera counts on CMU Panoptic \cite{panoptic}}
  Reported as the mean over all test frames, in GB.
  The flat enumerate-then-prune approach materializes the full $O(N^V)$ candidate space before applying the geometric threshold $\tau$ and runs out of memory on a 16 GB GPU at $V \ge 8$.
  Our construction prunes at every order and stays under 0.1 GB up to $V=10$.
  }
  \label{tab:memory_profile}
  \centering
  \small
  \setlength{\tabcolsep}{4pt}
  \begin{tabular*}{\textwidth}{@{\extracolsep{\fill}} @{\hspace{2pt}}l
  S[table-format=1.3] S[table-format=1.3] S[table-format=1.3]
  S[table-format=1.3] S[table-format=1.3] S[table-format=1.3]
  S[table-format=1.3] S[table-format=1.3] S[table-format=1.3] }
      \toprule
      \multirow{2}{*}{\textbf{Construction}}
          & \multicolumn{9}{c}{\textbf{Number of Views} $V$} \\
      \cmidrule(lr){2-10}
          & {2} & {3} & {4} & {5} & {6} & {7} & {8} & {9} & {10} \\
      \midrule
      Enumerate-Then-Prune     & 0.008 & 0.010 & 0.022 & 0.103 & 0.630 & 5.819 & {OOM} & {OOM} & {OOM} \\
      Lazily (Ours)      & 0.008 & 0.008 & 0.009 & 0.010 & 0.012 & 0.017 & 0.026 & 0.045 & 0.085 \\
      \bottomrule
  \end{tabular*}
\end{table}

%% file: assets/tables/panoptic_results_full.tex
\begin{table}[h!t]
    \caption{
    \textbf{Extended quantitative comparison on the CMU Panoptic dataset \cite{panoptic}.}
    Full version of \cref{tab:panoptic_main} including all reported fully supervised baselines.
    We report Average Precision (AP) at millimeter thresholds, Recall, and Mean Per Joint Position Error
    (MPJPE) in mm.
    $^{\dagger}$ uses 9 temporal frames as input.
    $^{\ddagger}$ uses the same 2D keypoint detector as our method.
    Best results per supervision category (full-, self-, and optimization-based) are highlighted in {\bestFull blue}, {\bestSelf orange}, and {\bestOpt green}.
    }
    \label{app_tab:panoptic_full}
    \centering
    \small
    \setlength{\tabcolsep}{3pt}
    \begin{tabular*}{\textwidth}{@{\extracolsep{\fill}} @{\hspace{2pt}}l
    S[table-format=2.2] S[table-format=2.2] S[table-format=2.2]
    S[table-format=3.2] S[table-format=2.2] S[table-format=3.2] }
        \toprule
        \multirow{2}{*}{\textbf{Method}}                                           & \multicolumn{4}{c}{\textbf{Average Precision (AP)} ($\uparrow$)} & {\textbf{Recall} ($\uparrow$)} & {\textbf{Error} ($\downarrow$)} \\
        \cmidrule(lr){2-5} \cmidrule(lr){6-6} \cmidrule(lr){7-7}
                                                                                   & {25}                             & {50}                           & {100}                          & {150}              & {@500}             & {MPJPE}            \\
        \midrule
        \multicolumn{7}{l}{\textit{Fully-Supervised}}                                      \\
        \hspace{1em} VoxelPose \cite{tu_VoxelPoseMulticamera3D_2020}                       & 83.59                            & 98.33                          & 99.76                          & 99.91              & {--}               & 17.68              \\
        \hspace{1em} Plane Sweep Pose \cite{lin_MultiViewMultiPerson3D_2021}               & 92.12                            & 98.96                          & \bestFull 99.81                & 99.84              & {--}               & 16.75              \\
        \hspace{1em} MvP \cite{wang_DirectMultiviewMultiperson_2021}                       & 92.28                            & 96.60                          & 97.45                          & 97.69              & {--}               & 15.76              \\
        \hspace{1em} Faster VoxelPose \cite{ye_FasterVoxelPoseRealtime_2022}               & 85.22                            & 98.08                          & 99.32                          & 99.48              & {--}               & 18.26              \\
        \hspace{1em} Wu \textit{et al.} \cite{wu_GraphBased3DMultiPerson_2021}             & 93.93                            & 98.93                          & 99.78                          & 99.90              & \bestFull 99.97    & 15.63              \\
        \hspace{1em} TEMPO \cite{choudhury_TEMPOEfficientMultiView_2023}                   & 89.01                            & \bestFull 99.08                & 99.76                          & \bestFull 99.93    & {--}               & 14.68              \\
        \hspace{1em} MVGFormer \cite{liao_MultipleViewGeometry_2024}                       & 92.32                            & 97.93                          & 99.32                          & 99.55              & 99.86              & 15.99              \\
        \hspace{1em} VoxelPose + 3DSA \cite{chen_3DSAMultiview3D_2025}                     & \bestFull 94.20                  & 98.49                          & 99.21                          & 99.31              & {--}               & \bestFull 13.98    \\
        \hspace{1em} MV-SSM \cite{chharia_MVSSMMultiViewState_2025}                        & 93.50                            & {--}                           & {--}                           & {--}               & {--}               & 15.70              \\
        \midrule
        \multicolumn{7}{l}{\textit{Self-Supervised}}                                       \\
        \hspace{1em} SelfPose3d \cite{srivastav_SelfPose3dSelfSupervisedMultiPerson_2024}  & 55.13                            & \bestSelf 96.44                & \bestSelf 98.46                & \bestSelf 98.98    & \bestSelf 99.60    & 24.47              \\
        \hspace{1em} DSP$^{\dagger}$ \cite{liu_DSPDenseSparseParallel_2025}                & \bestSelf 57.60                  & 86.10                          & 94.00                          & {--}               & {--}               & \bestSelf 23.10    \\
        \midrule
        \multicolumn{7}{l}{\textit{Optimization-Based}}                                    \\
        \hspace{1em} ACTOR \cite{pirinen_DomesDronesSelfSupervised_2019}                   & {--}                             & {--}                           & {--}                           & {--}               & {--}               & 168.40             \\
        \hspace{1em} MvPose$^{\ddagger}$ \cite{dong2019fast}                               & 37.63                            & 95.70                          & 97.84                          & 98.28              & 99.60              & 26.46              \\
        \hspace{1em} \name-ILP (Ours)                                                      & 66.70                            & 98.23                          & \bestOpt 99.43                 & \bestOpt 99.62     & \bestOpt 99.81     & \bestOpt 22.78     \\
        \hspace{1em} \name-BP (Ours)                                                       & \bestOpt 68.88                   & \bestOpt 98.37                 & 99.42                          & 99.61              & \bestOpt 99.81     & \bestOpt 22.78     \\
        \bottomrule
    \end{tabular*}
\end{table}

%% file: assets/tables/shelf_campus_results_full.tex
\begin{table}[h!t]
    \centering
    \caption[Quantitative Comparison Shelf and Campus (Full)]{
    \textbf{Extended quantitative comparison on the Shelf \cite{belagiannis_3DPictorialStructures_2014} and Campus \cite{belagiannis_3DPictorialStructures_2014} datasets.}
    Full version of \cref{tab:results_shelf_campus} including the additional fully supervised baselines that were omitted from the main paper for space.
    Best results per supervision category (full-, self-, and optimization-based) are highlighted in {\bestFull{blue}}, {\bestSelf{orange}}, and {\bestOpt{green}}.
    A1, A2, and A3 refer to Actor 1, 2, and 3, respectively.}
    \label{app_tab:results_shelf_campus_full}
    \small
    \setlength{\tabcolsep}{3pt}
    \begin{tabular*}{\textwidth}{ @{\extracolsep{\fill}} l *{8}{S[table-format=2.1]} @{} }
        \toprule                                                                                              & \multicolumn{4}{c}{\textbf{Shelf (PCP \%)} ($\uparrow$)} & \multicolumn{4}{c}{\textbf{Campus (PCP \%)} ($\uparrow$)} \\
        \cmidrule(lr){2-5} \cmidrule(l){6-9}
        \textbf{Method}                                                                                       & {A1}            & {A2}            & {A3}            & {Avg.}          & {A1}            & {A2}            & {A3}            & {Avg.}          \\
        \midrule
        \multicolumn{9}{l}{\textit{Fully Supervised}}                            \\
        \hspace{1em} VoxelPose \cite{tu_VoxelPoseMulticamera3D_2020}            & 99.3            & 94.1            & 97.6            & 97.0            & 97.6            & 93.8            & {\bestFull{98.8}} & 96.7            \\
        \hspace{1em} Wu et al. \cite{wu_GraphBased3DMultiPerson_2021}           & 99.3            & {\bestFull{96.5}} & 97.3            & {\bestFull{97.7}} & {--}            & {--}            & {--}            & {--}            \\
        \hspace{1em} MvP \cite{wang_DirectMultiviewMultiperson_2021}            & 99.3            & 95.1            & {\bestFull{97.8}} & 97.4            & {\bestFull{98.2}} & 94.1            & 97.4            & 96.6            \\
        \hspace{1em} Faster VoxelPose \cite{ye_FasterVoxelPoseRealtime_2022}    & {\bestFull{99.4}} & 96.0            & 97.5            & 97.6            & 96.5            & 94.1            & 97.9            & 96.2            \\
        \hspace{1em} TEMPO \cite{choudhury_TEMPOEfficientMultiView_2023}        & 99.3            & 95.1            & {\bestFull{97.8}} & 97.4            & 97.7            & {\bestFull{95.5}} & 97.9            & {\bestFull{97.3}} \\
        \midrule
        \multicolumn{9}{l}{\textit{Self-Supervised}}                             \\
        \hspace{1em} SelfPose3d                                                 & {\bestSelf{97.2}} & {\bestSelf{90.3}} & {\bestSelf{97.9}} & {\bestSelf{95.1}} & {\bestSelf{92.5}} & {\bestSelf{82.2}} & {\bestSelf{89.2}} & {\bestSelf{87.9}} \\
        \midrule
        \multicolumn{9}{l}{\textit{Optimization-Based}}                          \\
        \hspace{1em} 3DPS \cite{belagiannis_3DPictorialStructures_2014}         & 75.3            & 69.7            & 87.6            & 77.5            & 93.5            & 75.7            & 84.4            & 84.5            \\
        \hspace{1em} MvPose \cite{dong2019fast}                                 & 98.8            & {\bestOpt{94.1}}  & {\bestOpt{97.8}}  & {\bestOpt{96.9}}  & 97.6            & 93.3            & 98.0            & 96.3            \\
        \hspace{1em} \name-ILP (Ours)                                           & {\bestOpt{99.8}}  & 92.4            & 96.3            & 96.2            & {\bestOpt{99.4}}  & {\bestOpt{94.3}}  & {\bestOpt{98.1}}  & {\bestOpt{97.3}}  \\
        \hspace{1em} \name-BP (Ours)                                            & {\bestOpt{99.8}}  & 92.4            & 96.3            & 96.2            & {\bestOpt{99.4}}  & {\bestOpt{94.3}}  & 93.6            & 95.7            \\
        \bottomrule
    \end{tabular*}
\end{table}

%% file: supplementary/4_epilogue.tex
\section{Broader Societal Impact}
\label{app_sec:societal_impact}

This work advances optimization-based multi-view multi-person 3D pose estimation without requiring 3D supervision. 
It may benefit applications such as sports analysis, healthcare, operating-room monitoring, and human-robot collaboration, where calibrated multi-camera setups are common and accurate spatial localization is crucial. 
However, 3D human pose estimation raises privacy concerns, including biometric profiling and potential surveillance misuse. 
These risks are partly mitigated by the controlled acquisition setting required by \textsc{\name}, which assumes multiple synchronized calibrated cameras and is therefore less suited to unconstrained public deployment.